\def\year{2019}\relax
\documentclass[letterpaper]{article} 
\usepackage{aaai19}  
\usepackage{times}  
\usepackage{helvet}  
\usepackage{courier}  
\usepackage{graphicx}  
\usepackage{url}            
\usepackage{booktabs}       
\usepackage{amsfonts}       
\usepackage{nicefrac}       
\usepackage{microtype}      

\usepackage{subfigure}
\usepackage{algorithm}
\usepackage{algorithmic}
\usepackage{amsmath}
\usepackage{amsfonts}
\usepackage{amssymb}
\usepackage{amsthm}
\newtheorem{theorem}{Theorem}
\newtheorem{corollary}{Corollary}
\newtheorem{lemma}{Lemma}

\newtheorem{proposition}{Proposition}
\newtheorem{definition}{Definition}

\newtheorem{innercustomthm}{Theorem}
\newenvironment{customthm}[1]
  {\renewcommand\theinnercustomthm{#1}\innercustomthm}
  {\endinnercustomthm}

\newtheorem{innercustomcor}{Corollary}
\newenvironment{customcor}[1]
  {\renewcommand\theinnercustomcor{#1}\innercustomcor}
  {\endinnercustomcor}

\newtheorem{innercustomprop}{Proposition}
\newenvironment{customprop}[1]
  {\renewcommand\theinnercustomprop{#1}\innercustomprop}
  {\endinnercustomprop}

\begin{document}
%
\title{Machine Teaching for Inverse Reinforcement Learning: \\
Algorithms and Applications}
\author{ Daniel S. Brown \and   Scott Niekum \\
 Department of Computer Science \\
  University of Texas at Austin \\
  \texttt{\{dsbrown,sniekum\}@cs.utexas.edu} \\
}
\maketitle
\begin{abstract}
Inverse reinforcement learning (IRL) infers a reward function from demonstrations, allowing for policy improvement and generalization. However, despite much recent interest in IRL, little work has been done to understand the minimum set of demonstrations needed to teach a specific sequential decision-making task. We formalize the problem of finding maximally informative demonstrations for IRL as a machine teaching problem where the goal is to find the minimum number of demonstrations needed to specify the reward equivalence class of the demonstrator. We extend previous work on algorithmic teaching for sequential decision-making tasks by showing a reduction to the set cover problem which enables an efficient approximation algorithm for determining the set of maximally-informative demonstrations. We apply our proposed machine teaching algorithm to two novel applications: providing a lower bound on the number of queries needed to learn a policy using active IRL and developing a novel IRL algorithm that can learn more efficiently from informative demonstrations than a standard IRL approach.
\end{abstract}

\section{Introduction}
As robots and digital personal assistants become more prevalent, there is growing interest in developing algorithms that allow everyday users to program or adapt these intelligent systems to accomplish sequential decision-making tasks, such as performing household chores, or carrying on a meaningful conversation. A common way to teach sequential decision-making tasks is through Learning from Demonstration (LfD) \cite{Argall2009}, in which the goal is to learn a policy from demonstrations of desired behavior.  More specifically, Inverse Reinforcement Learning (IRL) \cite{ng2000algorithms,arora2018survey} is a form of LfD that aims to infer the reward function that motivated the demonstrator's behavior, allowing for reinforcement learning \cite{sutton1998reinforcement} and generalization to unseen states.
Despite much interest in IRL, there is not a clear, agreed-upon definition of optimality in IRL, namely, the size of the minimal set of demonstrations needed to teach a sequential decision-making task. 

There are many compelling reasons to study optimal teaching for IRL: (1)  it gives insights into the intrinsic difficulty of teaching certain sequential decision-making tasks; (2) it provides a lower bound on the number of samples needed by active IRL algorithms \cite{lopes2009active,brown2018risk}; (3) optimal teaching can be used to design algorithms that better leverage highly informative demonstrations which do not follow the i.i.d. assumption made by many IRL algorithms; (4) studying optimal teaching can help humans better teach robots through demonstration \cite{cakmak2012algorithmic} and help robots better communicate their intentions \cite{huang2017enabling}; and (5) optimal teaching can give insight into how to design \cite{mei2015using} and defend against \cite{alfeld2017explicit} demonstration poisoning attacks in order to design IRL algorithms that are robust to poor or malicious demonstrations.

We formulate the problem of optimal teaching for sequential decision making tasks using the recently popularized \textit{machine teaching} framework \cite{zhu2015machine}. 
The machine teaching problem is the inverse of the machine learning problem. In machine teaching, the goal is to select the optimal training set that minimizes teaching cost, often defined as the size of the training data set, and the loss or teaching risk between the model learned by the student and the learning target. While machine teaching has been applied to regression and classification \cite{zhu2015machine,liu2016teaching}, little work has addressed machine teaching for sequential decision-making tasks such as learning from demonstration via IRL.

The contributions of this paper are fourfold: (1) a formal definition of machine teaching for IRL, (2) an efficient algorithm to compute optimal teaching demonstrations for IRL, (3) an application of machine teaching to find the lower bound on the number of queries needed to learn a task using active IRL, and (4) a novel Bayesian IRL algorithm that learns more efficiently from informative demonstrations than a standard IRL approach by leveraging the non-i.i.d. nature of highly informative demonstrations from a teacher.

\section{Related work}

Determining the minimum number of demonstrations needed to teach a task falls under the fields of Algorithmic Teaching \cite{goldman1995complexity,balbach2009recent} and Machine Teaching \cite{zhu2015machine,zhu2018overview}. However, almost all previous work has been limited to optimal teaching for classification and regression tasks. The work of Singla et al. \shortcite{singla2014near} bears a strong resemblance to our work: they use submodularity to find an efficient approximation algorithm for an optimal teaching problem that has a set-cover reduction; however, their approach is designed for binary classification rather than sequential decision making.

Cakmak and Lopes \shortcite{cakmak2012algorithmic} examined the problem of giving maximally informative demonstrations to teach a sequential decision-making task; however, as we discuss in Section~\ref{sec:Cakmak}, their algorithm often underestimates the minimum number of demonstrations needed to teach a task. Other related approaches examine how a robot can give informative demonstrations to a human \cite{huang2017enabling}, or formalize optimal teaching as a cooperative two-player Markov game \cite{hadfield2016cooperative}; however, neither approach addresses the machine teaching problem of finding the minimum number of demonstrations needed to teach a task. 

Our proposed machine teaching algorithm leverages the notion of behavioral equivalence classes over reward functions to achieve an efficient approximation algorithm. Zhang et al. \shortcite{zhang2009policy} also use behavioral equivalence classes over reward functions as part of their solution to a policy teaching problem, in which the goal is to induce a desired policy by modifying the intrinsic reward of an agent through incentives. Rathnasabapathy et al. \shortcite{rathnasabapathy2006exact} and Zeng et al. \shortcite{zeng2012exploiting} use equivalence classes over agent behaviors when solving the problem of interacting with multiple agents in a POMDP.

There is a large body of work on using active learning for IRL \cite{lopes2009active,cohn2011comparing,cuiactive2017,sadigh2017active,brown2018risk}.
Our goal of finding a minimal set of demonstrations to teach an IRL agent is related to one of the goals of active learning: reducing the number of examples needed to learn a concept \cite{settles2012active}. In active learning, the agent requests labeled examples to search for the correct hypothesis. Optimal teaching is usually more sample efficient than active learning since the teacher gets to pick maximally informative examples to teach the target concept to the learner \cite{zhu2018overview}. Thus, a solution to the machine teaching problem for IRL provides a method for finding the lower bound on the number of queries needed to learn a policy when using active IRL.  

In the field of Cognitive Science, researchers have investigated Bayesian models of informative human teaching and the inferences human students make when they know they are being taught \cite{shafto2008teaching}.  Ho et al. \cite{ho2016showing} showed that humans give different demonstrations when performing a sequential decision making task, depending on whether they are teaching or simply doing the task. While studies have shown that standard IRL algorithms can benefit from informative demonstrations \cite{cakmak2012algorithmic,ho2016showing}, to the best of our knowledge, no IRL algorithms exist that can explicitly leverage the informative nature of such demonstrations.
 In Section~\ref{subsec:BIO-IRL} we propose a novel IRL algorithm that can learn more efficiently from informative demonstrations than a standard Bayesian IRL approach that assumes demonstrations are drawn i.i.d. from the demonstrator's policy. Research in computational learning theory has shown a dramatic reduction in the number of teaching examples needed to teach anticipatory learners who know they are being taught by a teacher \cite{doliwa2014recursive,gao2017preference}, but has not addressed sequential decision making tasks. To the best of our knowledge, our work is the first to demonstrate the advantages of an anticipatory IRL algorithm which can leverage the non-i.i.d. nature of highly informative demonstrations from a teacher.

\section{Problem formalism}
\subsection{Markov decision processes}
We model the environment as a Markov decision process (MDP), $\langle \mathcal{S}, \mathcal{A}, T, R, \gamma, S_0 \rangle$, where $\mathcal{S}$ is the set of states, $\mathcal{A}$ is the set of actions, $T:\mathcal{S} \times \mathcal{A} \times \mathcal{S} \to [0,1]$ is the transition function, $R: \mathcal{S} \to \mathbb{R}$ is the reward function, $\gamma \in [0,1)$ is the discount factor, and $S_0$ is the initial state distribution. 
A policy $\pi : \mathcal{S} \times \mathcal{A} \mapsto [0,1]$ is a mapping from states to a probability distribution over actions. 
We assume that a stochastic optimal policy gives equal probability to all optimal actions. 
The value of executing policy $\pi$ starting at state $s \in S$ is defined as 
\begin{equation}
V^\pi(s) = \mathbb{E}[\sum_{t=0}^\infty \gamma^t R(s_t) \mid \pi, s_0 = s].
\end{equation}
The Q-value of a state-action pair $(s,a)$ is defined as 
\begin{equation}
Q^{\pi}(s,a) = R(s) + \gamma \mathbb{E}_{s' \sim T(\cdot\mid s,a)}[ V^\pi(s')]
\end{equation}
and we denote the optimal Q-value function as $Q^*(s,a) = \max_{\pi} Q^{\pi}(s,a)$.

As is common in the literature \cite{ziebart2008maximum,sadigh2016information,pirotta2016inverse,barreto2017successor}, we assume that the reward function can be expressed as a linear combination of features, $\phi: \mathcal{S} \mapsto \mathbb{R}^k$, so that $R(s) = \mathbf{w}^T \phi(s)$ where $\mathbf{w} \in \mathbb{R}^k$ is the vector of feature weights. This assumption is not restrictive as these features can be nonlinear functions of the state variables. We can write the expected discounted return of a policy as  
\begin{eqnarray}
\rho(\pi) \,= \,\mathbb{E}[\sum_{t=0}^{\infty} \gamma^t \mathbf{w}^T \phi(s_t) \mid \pi]
\,=\, \mathbf{w}^T \mu_{\pi},
\end{eqnarray}
where 
$\mu_\pi = \mathbb{E}[\sum_{t=0}^{\infty}\gamma^t \phi(s_t) | \pi]$.

\subsection{Machine teaching}
The machine teaching problem \cite{zhu2015machine} is to select the optimal training set $D^*$ that minimizes the teaching cost, often defined as the size of the data set, and the teaching risk which represents the teacher's dissatisfaction with the model learned by the student. We focus on the constrained form of machine teaching \cite{zhu2018overview} defined as 
\begin{eqnarray}
\min_D& &\text{TeachingCost}(D) \\
s.t.& &\text{TeachingRisk}(\hat{\theta}) \leq \epsilon \\
& &\hat{\theta} = \text{MachineLearning}(D)
\end{eqnarray}
where $D$ is the training set to be optimized, $\hat{\theta}$ is the model the student learns under $D$, and $\epsilon \geq 0$ determines how much the model learned by the student can differ from the learning target of the teacher.

\subsection{Problem definition} \label{subsec:problemdef}
We now formulate the optimal teaching problem for IRL as a machine teaching problem. 
We assume that the expert teacher operates under a ground-truth reward, $R^*$, and is able to demonstrate state-action pairs $(s,a)$ by executing the corresponding optimal policy $\pi^*$. A naive formulation of the machine teaching problem for IRL would be to find the minimal set of demonstrations, $\mathcal{D}$, that enables an IRL agent to learn $R^*$ within some $\epsilon$ teaching risk.
However, IRL is ill-posed \cite{ng2000algorithms}---there are an infinite number of reward functions that explain any optimal policy. Instead, we focus on determining the minimal set of demonstrations that enable a learner to find a reward function that results in an optimal policy with performance similar to the performance of the teacher's policy under $R^*$. Specifically, we define the policy loss of an estimated weight vector $\mathbf{\hat{w}}$ compared with the true weight vector $\mathbf{w^*}$ as
\begin{equation}\label{eq:loss}
\text{Loss}(\mathbf{w^*}, \mathbf{\hat{w}}) = \mathbf{w^*}^T ( \mu_{\pi^*} - \mu_{\hat{\pi}}),
\end{equation}
where $\pi^*$ is the optimal policy under $\mathbf{w}^*$ and $\hat{\pi}$ is the optimal policy under $\mathbf{\hat{w}}$. Equation~(\ref{eq:loss}) gives the difference in expected return between the teacher's policy $\pi^*$ and the expected return of the learner's policy, when both are evaluated under the teacher's reward function $R^* = \mathbf{w^*}^T \phi(s)$. We can now formalize the machine teaching problem for IRL. 

\paragraph{Machine teaching problem for IRL:} Given an MDP, $\mathcal{M}$, and the teacher's reward function, $R^* = \mathbf{w^*}^T \phi(s)$,  find the set of demonstrations, $\mathcal{D}$, that minimizes the following optimization problem:
\begin{eqnarray}
\min_\mathcal{D}& &\text{TeachingCost}(\mathcal{D}) \\
s.t.& &\text{Loss}(\mathbf{w^*}, \mathbf{\hat{w}}) \leq \epsilon \\
& &\mathbf{\hat{w}} = \text{IRL}(\mathcal{D})
\end{eqnarray}
where $\mathcal{D}$ is the set of demonstrations, 
and $\mathbf{\hat{w}}$ is the reward recovered by the learner using Inverse Reinforcement Learning (IRL). 
This formalism covers both exact teaching ($\epsilon = 0$) and approximate teaching ($\epsilon > 0$).
In this work we define 
\begin{equation}
\text{TeachingCost}(\mathcal{D}) = |\mathcal{D}| 
\end{equation}
where, $|\mathcal{D}|$ can denote either the number of $(s,a)$ pairs in $\mathcal{D}$ or the number of trajectories in $\mathcal{D}$; however, our proposed approach can be easily extended to problems with different teaching costs, e.g., where some demonstrations may be more expensive or dangerous for the teacher. 

\subsection{Discussion}
Like most machine teaching problems \cite{zhu2018overview}, the machine teaching problem for IRL is a difficult optimization problem. A brute-force approach would require searching over the power set of all possible demonstrations. This search is intractable due to the size of the power set and the need to solve an IRL problem for each candidate set of demonstrations. 
One of our contributions is an efficient algorithm for solving the machine teaching problem for IRL that only requires solving a single policy evaluation problem to find the expected feature counts of $\pi^*$ and then running a greedy set-cover approximation algorithm.

Before discussing our proposed approach in detail, we first introduce the notion of a behavioral equivalence class which is a key component of our approximation algorithm. We will also provide an overview and analysis of the work of \citeauthor{cakmak2012algorithmic} \shortcite{cakmak2012algorithmic} which provides the baseline and motivation for our approach.

\section{Behavioral Equivalence Classes} \label{sec:FeasibleRegion}

The \textit{behavioral equivalence class} (BEC) of a policy $\pi$ is defined as the set of reward functions under which $\pi$ is optimal:
\begin{equation}\label{eq:bec}
\text{BEC}(\pi) =  
\{\mathbf{w} \in \mathbb{R}^k \mid \pi \text{ optimal w.r.t. } R(s) = \mathbf{w}^T \phi(s) \}. 
\end{equation}
In this section we briefly discuss how to calculate the behavioral equivalence class for both a policy and for a set of demonstrations from a policy.
Given an MDP with either finite or continuous states and with a reward function represented as a linear combination of features, Ng and Russell \shortcite{ng2000algorithms} derived the behavioral equivalence class (BEC) for a policy. We summarize their result as follows:


\begin{theorem} \label{thm:contFeasible} \cite{ng2000algorithms}
Given an MDP, BEC($\pi$) is given by the following intersection of half-spaces:
\begin{eqnarray}
&&\mathbf{w}^T (\mu_{\pi}^{(s,a)} - \mu_{\pi}^{(s,b)}) \geq 0,\\
&&\forall a \in \arg\max_{a'\in \mathcal{A}} Q^*(s,a'), b \in \mathcal{A}, s \in \mathcal{S},
\end{eqnarray}
where $\mathbf{w} \in \mathbb{R}^k$ are the reward function weights and 
\begin{equation}\mu_{\pi}^{(s,a)} = \mathbb{E}[\sum_{t=0}^{\infty} \gamma^t \phi(s_t) | \pi, s_0 = s, a_0 = a],
\end{equation}
 is the vector of expected feature counts that result from taking action $a$ in state $s$ and following $\pi$ thereafter. 
\end{theorem}

We can similarly define the BEC for a set of demonstrations $\mathcal{D}$ from a policy $\pi$:

\begin{corollary}\label{cor:feasibleDemo}
BEC$(\mathcal{D} | \pi)$ is given by the following intersection of half-spaces:
\begin{equation}
\mathbf{w}^T (\mu_{\pi}^{(s,a)} - \mu_{\pi}^{(s,b)}) \geq 0,\; \forall (s,a) \in \mathcal{D}, b \in \mathcal{A}.  
\end{equation}
\end{corollary}
All proofs can be found in the appendix.

\paragraph{Example:}
Consider the grid world shown in Figure~\ref{subfig:simpleNg}, with four actions available in each state and deterministic transitions. We computed the BEC using a featurized reward function $R(s) = \mathbf{w}^T \phi(s)$, where $\mathbf{w} = (w_0, w_1)$ with $w_0$ indicating the reward weight for a ``white" cell and $w_1$ indicating the reward weight for the ``grey" cell (see appendix for full details).
\begin{figure}[t]
\centering
\subfigure[Policy]{\includegraphics[scale=0.09]{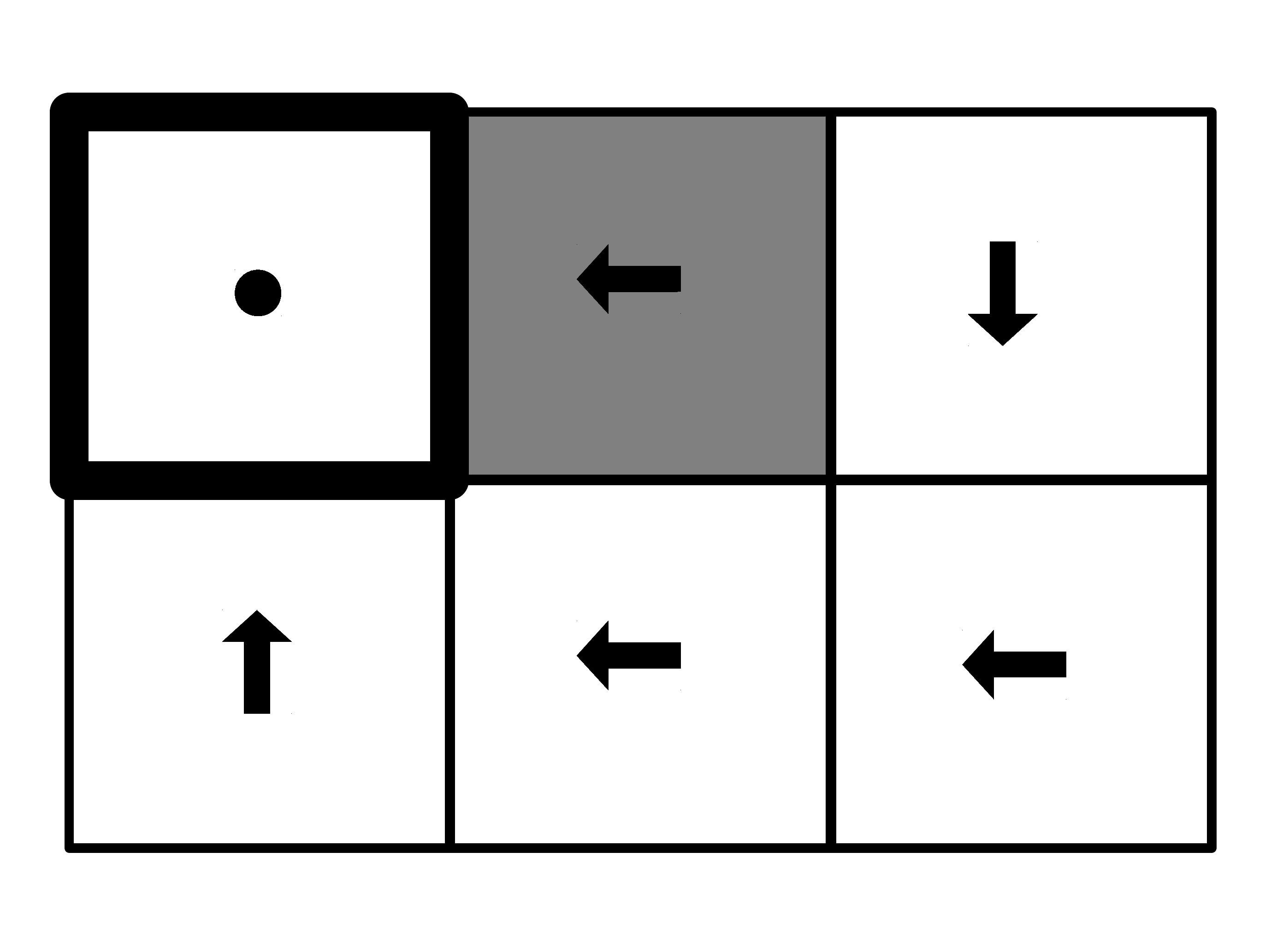}
\label{subfig:simpleNg}
}
\subfigure[BEC(Policy)]{
\includegraphics[scale=0.18]{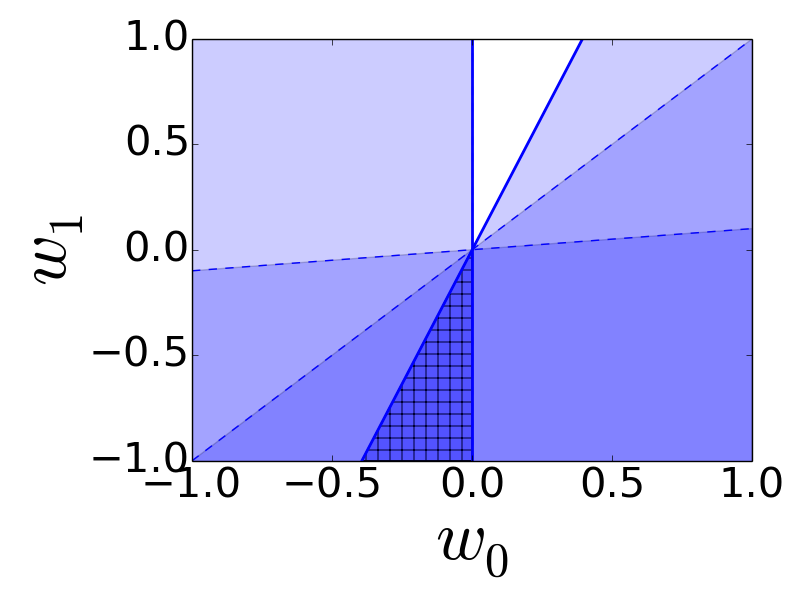}
\label{subfig:simplePolcicyFeasible}
}
\subfigure[Demonstration]{\includegraphics[scale=0.09]{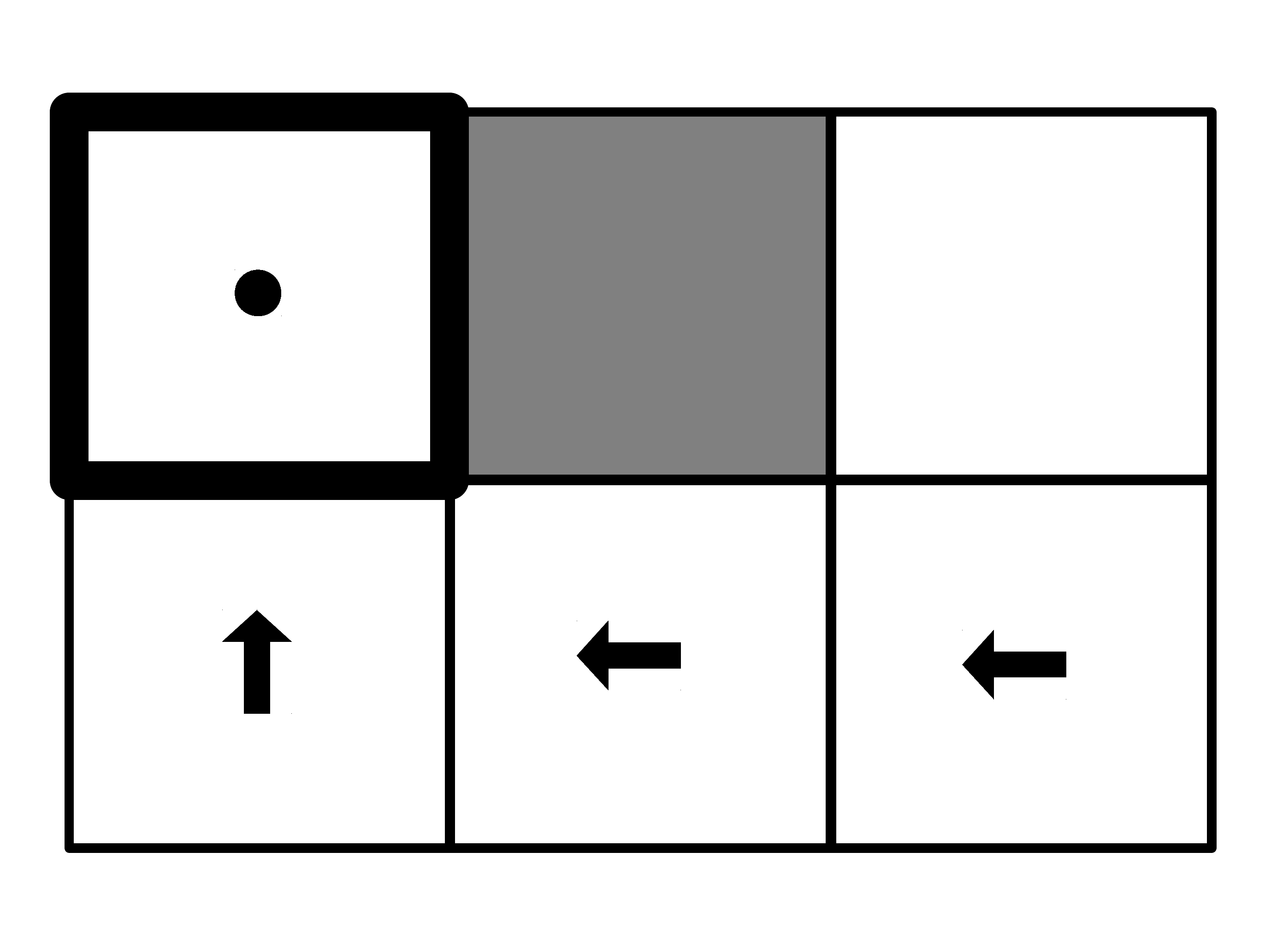}
\label{subfig:simpleCakmak}
}
\subfigure[BEC(Demonstration)]{
\includegraphics[scale=0.18]
{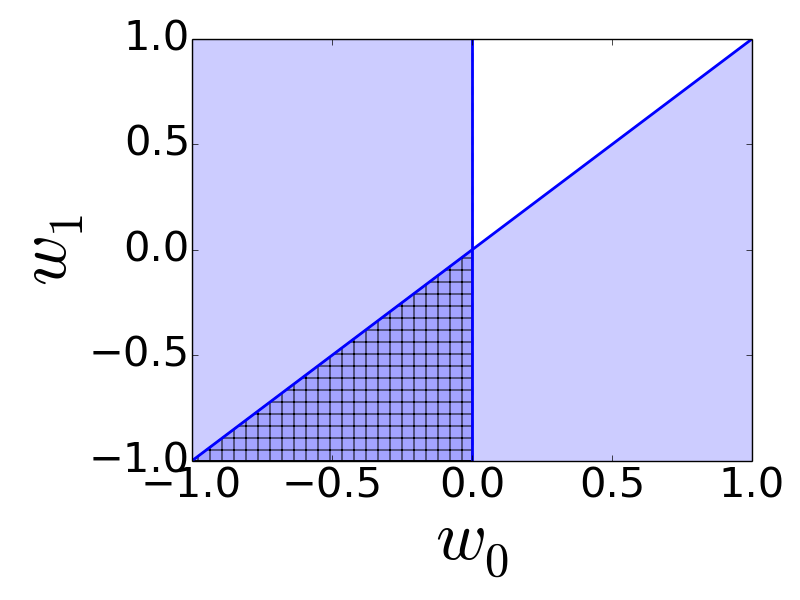}
\label{subfig:simpleDemoFeasible}
}
\caption{Behavioral equivalence classes (BEC) for a grid world with 6 states. The upper left state is a terminal state. Each state has 4 available actions and the reward function is a linear combination of two binary features that indicate whether the color of the cell is white or grey. (a) The optimal policy. (b) The resulting intersection of half-spaces (shaded region) that defines all weight vectors such that the policy shown in (a) is optimal. (c) A demonstration from the optimal policy in (a) that is not maximally informative. (d)  The intersection of half-spaces (shaded region) resulting from the demonstration shown in (c).}
\label{fig:exampleFeasiblePolicy}
\end{figure}
The resulting half-space constraints are shown in Figure~\ref{subfig:simplePolcicyFeasible}. The intersection of these half-spaces exactly describes the set of rewards that make the policy shown in Figure~\ref{subfig:simpleNg} optimal: both white and grey cells have negative reward and the weight for the grey feature is low enough that the optimal policy avoids the shaded cell when starting from the top right cell.

Figures~\ref{subfig:simpleCakmak} and \ref{subfig:simpleDemoFeasible} show the BEC of a demonstration. 
The demonstration shows that both feature weights are non-positive and that $w_1$ is no better than $w_0$ (otherwise the demonstration would have gone through the grey cell); however, the demonstration leaves open the possibility that all feature weights are equal. However, if the demonstration had started in the top right cell, the BEC of the demonstration would be identical to the BEC of the optimal policy. This highlights the fact that some demonstrations from an optimal policy are more informative than others. An efficient algorithm for finding maximally informative demonstrations using the BEC of the teachers policy is one of the contributions of this paper.



\section{Uncertainty Volume Minimization} \label{sec:Cakmak}
We now give an overview of the algorithmic teaching approach proposed by Cakmak and Lopes \shortcite{cakmak2012algorithmic} which motivates our work. 
The main insight that Cakmak and Lopes use is that of Corollary~\ref{cor:feasibleDemo}: if an optimal demonstration contains $(s,a)$, then an IRL algorithm can infer that 
\begin{eqnarray}
&&Q^*(s,a) \geq Q^*(s,b), \forall b \in \mathcal{A} \\
\Leftrightarrow &&\mathbf{w}^T (\mu^{(s,a)}_{\pi^*} - \mu^{(s,b)}_{\pi^*}) \geq 0, \forall b \in \mathcal{A}.
\end{eqnarray}


Given a candidate demonstration set $\mathcal{D}$, Cakmak and Lopes use the intersection of the corresponding half-spaces (as defined in Corollary~\ref{cor:feasibleDemo}) as a representation of the learner's uncertainty over the true reward function. They use a Monte Carlo estimate of the volume of this cone as a measure of the learner's uncertainty, and seek demonstrations that minimize
 the uncertainty, $G(\mathcal{D})$, over the true reward function, where
\begin{equation}
G(\mathcal{D}) = \frac{1}{N} \sum_{j=1}^N \delta(x_j \in C(\mathcal{D})),
\end{equation}
$\delta$ is an indicator function, $C(\mathcal{D})$ is the intersection of half-spaces given in Corollary~\ref{cor:feasibleDemo}, and the volume is estimated by drawing $N$ random points $x_j$ from $[-1,1]^k$.
The set of maximally informative demonstrations is chosen greedily by iteratively selecting the trajectory that maximally decreases $G(\mathcal{D})$. This process repeats until $G(\mathcal{D})$ falls below a user defined threshold $\epsilon$. We refer to this algorithm as the Uncertainty Volume Minimization (UVM) algorithm.


In the UVM algorithm, trajectories are added until $G(\mathcal{D})$ is below some user-provided threshold $\epsilon$; however, this does not solve the machine teaching problem for IRL presented in Section~\ref{subsec:problemdef}. This is because it only ensures that the estimated uncertainty volume $G(\mathcal{D})$ is less than $\epsilon$, not that the policy loss (Eq.~(\ref{eq:loss})) is less than $\epsilon$. In order to guarantee that the policy loss is below a desireable level, this threshold must be carefully tuned for every MDP. If $\epsilon$ is too low, then the algorithm will never terminate.
Alternatively, if $\epsilon$ is too high, then not enough demonstrations will be selected to teach an appropriate reward function and the policy loss may be large, depending on the reward function selected by the IRL algorithm. 
In our experiments, we remove the need for parameter tuning by stopping the UVM algorithm if it cannot find a new demonstration to add that decreases $G(\mathcal{D})$.

Another limitation of the UVM algorithm is that of volume estimation: exact volume estimation is $\#\mathcal{P}$-hard \cite{valiant1979complexity,simonovits2003compute} and straightforward Monte Carlo estimation is known to fail in high-dimensions \cite{simonovits2003compute}. 
Additionally, if there are two (or more) actions that are both optimal in state $s$, and those two actions ($a$ and $b$) are demonstrated, this will result in the following constraints:
\begin{eqnarray}
&w^T(\mu^{(s,a)}_\pi - \mu^{(s,b)}_\pi) \geq 0 \text{\; and \;}
w^T(\mu^{(s,b)}_\pi - \mu^{(s,a)}_\pi) \geq 0& \nonumber \\
&\Rightarrow w^T(\mu^{(s,b)}_\pi - \mu^{(s,a)}_\pi) = 0&
\end{eqnarray}
This is problematic because any strict subspace of $\mathbb{R}^k$ has measure zero, resulting in an uncertainty volume of zero. Thus, the UVM algorithm will terminate with zero uncertainty if two optimal actions are ever demonstrated from the same state, even if this leaves an entire $(k-1)$ dimensional subspace of uncertainty over the reward function. Note that when $\pi^*$ is a stochastic optimal policy this behavior will occur once any trajectory that contains a state with more than one optimal action is chosen---the best trajectory to select next will always be one that visits a previously demonstrated state and chooses an alternative optimal action, resulting in zero uncertainty volume. 

One possible solution would be to use more sophisticated sampling techniques such as hit-and-run sampling \cite{smith1984efficient}. However, these methods still assume there are points on the interior of the sampling region and can be computationally intensive, as they require running long Markov chains to ensure good mixing.
Another possible remedy is to only use deterministic policies for teaching.
We instead propose a novel approach based on a set cover equivalence which removes the need to estimate volumes and works for both deterministic and stochastic teacher policies.

\section{Set Cover Machine Teaching for IRL}\label{sec:MachineTeachingAlgo}
Our proposed algorithm seeks to remedy the problems with the UVM algorithm identified in the previous section in order to find an efficient approximation to the machine teaching problem proposed in Section~\ref{subsec:problemdef}.

Our first insight is the following:
\begin{proposition} \label{prop:1}
Consider an optimal policy $\pi^*$ for reward $R^*(s) = \mathbf{w^*}^T \phi(s)$. Given any weight vector $\mathbf{w} \in$ BEC($\pi^*$), if $R(s) = \mathbf{w}^T\phi(s)$ is not constant for all states in $\mathcal{S}$, then Loss($\mathbf{w^*}$, $\mathbf{w}$) = 0.
\end{proposition}

This proposition says that if we have a non-degenerate weight vector in the behavioral equivalence class for a policy $\pi^*$, then we incur zero policy loss by using $\mathbf{w}$ rather than $\mathbf{w^*}$ for performing policy optimization. This follows directly from Equation~(\ref{eq:bec}). Thus, to ensure that the policy loss constraint, Loss$(\mathbf{w^*}, \mathbf{\hat{w}})\leq \epsilon$, holds in the machine teaching problem, we can focus on finding a demonstration set $\mathcal{D}$ such that the weight vector, $\mathbf{\hat{w}}$, learned through IRL is in BEC($\pi^*$). 

Note that Proposition~\ref{prop:1} also assumes that the IRL agent being taught will not find a degenerate reward function if a non-degenerate solution exists. This property is true of all standard IRL methods \cite{gao2012survey,arora2018survey}. While it is possible that an IRL algorithm may return a constant reward function (e.g., $R(s) = 0, \forall s \in \mathcal{S}$) if $\mathcal{D} = \emptyset$, the only way for $\mathcal{D} = \emptyset$ to be the optimal solution for machine teaching is if the resulting loss is less than $\epsilon$, i.e.,
\begin{equation}
\text{Loss}(\mathbf{w^*}, \mathbf{\hat{w}}) = \mathbf{w^*}^T ( \mu_{\pi^*} - \mu_{\hat{\pi}})\leq \epsilon.
\end{equation}
For $\epsilon=0$, this will be false since a constant reward function will almost surely lead to an optimal policy $\hat{\pi}$ which does not match the feature counts of the teacher's policy, $\pi^*$. 

Our second insight is based on the fact that the behavioral equivalence class for $\pi^*$ is an intersection of half-spaces (Theorem~\ref{thm:contFeasible}). Rather than give demonstrations until the uncertainty volume, $G(\mathcal{D})$, is less than some arbitrary value, demonstrations should be chosen specifically to define BEC($\pi^*$). Thus, to obtain a feasible solution to the machine teaching problem for IRL we need to select a demonstration set such that the corresponding intersection of half-spaces, BEC$(\mathcal{D}|\pi^*)$ is equal to BEC$(\pi^*)$.

Our final insight is to formulate an efficient approximation algorithm for the machine teaching problem for IRL through a reduction to the set cover problem. This allows us to avoid the difficult volume estimation problem required by the UVM algorithm and focus instead on a well known discrete optimization problem. From Section~\ref{sec:FeasibleRegion} we know that the behavioral equivalence class of both a policy and a demonstration are both characterized by intersections of half-spaces, and each demonstration from $\pi^*$ produces an intersection of half-spaces which contains BEC$(\pi^*)$. Thus, the machine teaching problem for IRL (Section~\ref{subsec:problemdef}) with $\epsilon = 0$ is an instance of the set cover problem: we have a set of half-spaces defining BEC$(\pi^*)$, each possible trajectory from $\pi^*$ covers zero or more of the half-spaces that define BEC$(\pi^*)$, and we wish to find the smallest set of demonstrations, $\mathcal{D}$, such that $\text{BEC}(\mathcal{D} | \pi^*) = \text{BEC}(\pi^*)$. 

One potential issue is that, as seen in Figure~\ref{fig:exampleFeasiblePolicy}, many half-space constraints will be non-binding and we are only interested in covering the non-redundant half-space constraints that minimally define BEC$(\pi^*)$. To address this, we use linear programming to efficiently remove redundant half-spaces constraints \cite{paulraj2010comparative} before running our set cover algorithm (see appendix for details).

Note that this approach allows us to solve the machine teaching IRL problem without needing to repeatedly solve RL or IRL problems. The only full RL problem that needs to be solved is to obtain $\pi^*$ from $\mathbf{w^*}$. After $\pi^*$ is obtained, we can efficiently solve for the feature expectations $\mu_{\pi^*}^{(s,a)}$ by solving the following equation 
\begin{equation}
\mu_{\pi^*}^{(s,a)} = \phi(s) + \gamma \mathbb{E}_{s'|a}[\mu_{\pi^*}^{(s')}] 
\end{equation}
where 
$\mu_{\pi^*}^{(s)} = \phi(s) + \gamma \mathbb{E}_{s'|\pi^*(s)} [\mu_{\pi^*}^{(s')}]$.
These equations satisfy a Bellman equation and can be solved for efficiently. The values $\mu_{\pi^*}^{(s,a)}$ and $\mu_{\pi^*}^{(s)}$ are often called successor features \cite{dayan1993improving} in reinforcement learning and recent work has shown that they can be efficiently computed for model-free problems with continuous state-spaces \cite{barreto2017successor}.

We summarize our approach as follows: Given $\pi^*$, the optimal policy under the teachers reward function $\mathbf{w^*}$, (1) Solve for the successor features $\mu_{\pi^*}^{(s,a)}$,  (2) Find the half-space constraints for BEC($\pi^*$) using Theorem~\ref{thm:contFeasible}, (3) Find the minimal representation of BEC($\pi^*$) using linear programming, (4) Generate candidate demonstrations under $\pi^*$ from each starting state and calculate their corresponding half-space unit normal vectors using Corollary~\ref{cor:feasibleDemo}, and (5) Greedily cover all half-spaces in BEC($\pi^*$) by sequentially picking the candidate demonstration that covers the most uncovered half-spaces. 

We call this algorithm Set Cover Optimal Teaching (SCOT) and give pseudo-code in Algorithm~\ref{alg:scot}. In the pseudo-code we use $\mathbf{\hat{N}}[\cdot]$ to denote the set of unit normal vectors for a given intersection of half-spaces and $\setminus$ to denote set subtraction. To generate candidate demonstration trajectories we perform $m$ rollouts of $\pi^*$ for each start state. If $\pi^*$ is deterministic, then $m=1$ is sufficient.

Whereas UVM finds demonstrations that successively slice off volume from the uncertainty region,  SCOT directly estimates the minimal set of demonstrations that exactly constrain BEC$(\pi^*)$. This removes both the need to calculate high-dimensional volumes and the need to determine an appropriate stopping threshold. 
SCOT also has the following desirable properties:

\begin{proposition}
The Set Cover Optimal Teaching (SCOT) algorithm always terminates.
\end{proposition}




\begin{theorem}
Under the assumption of error-free demonstrations, SCOT is a $(1-1/e)$-approximation to the Machine Teaching Problem for IRL (Section~\ref{subsec:problemdef}) for the following learning algorithms: Bayesian IRL \cite{ramachandran2007bayesian,choi2011map},
Policy Matching \cite{neu2007apprenticeship},
and Maximum Likelihood IRL \cite{babes2011apprenticeship,lopes2009active}.
\end{theorem}

\begin{algorithm}[t] 
\caption{Set Cover Optimal Teaching (SCOT)} 
\label{alg:scot} 
\begin{algorithmic}[1] 
    \REQUIRE MDP $\mathcal{M}$ with set of possible initial states $S_0$ and reward function $R^*(s) = \mathbf{w^*}^T\phi(s)$.
    \STATE  \textit{// Compute the behavioral equivalence class of $\pi^*$}
    \STATE Compute optimal policy $\pi^*$ for $\mathcal{M}$ and feature expectations $\mu_{\pi^*}^{(s,a)}$.
    \STATE Use Theorem 1 to compute BEC$(\pi^*)$.
    \STATE $U \gets \mathbf{\hat{N}}[\text{BEC}(\pi^*)]$.
    \STATE Remove redundant half-space constraints from $U$.
    \STATE  \textit{// Compute candidate demonstration trajectories}
    \STATE $\mathcal{T} = \emptyset$
    \FORALL{$s_0 \in S_0$}
    	\FOR{$i = 1,\ldots,m$}
    		\STATE Generate trajectory $\tau = (s_0, a_0, \ldots, s_{H-1}, a_{H-1})$ by starting at $s_0$ and following $\pi^*$ for $H$ steps.
    		\STATE $\mathcal{T} = \mathcal{T} \cup \tau$
    		\STATE Use Corollary~\ref{cor:feasibleDemo} to calculate BEC$(\tau | \pi^*)$
    	\ENDFOR
    \ENDFOR
    \STATE \textit{// Solve set cover using greedy approximation}
    \STATE $\mathcal{D} \gets \emptyset$, $C \gets \emptyset$
    \WHILE{$|U \setminus C| \neq 0$}
   		\STATE $\tau_{\rm greedy} = \arg\max_{\tau \in \mathcal{T}} \left| \mathbf{\hat{N}}[\text{BEC}(\tau|\pi^*)] \cap U \setminus C\right|$
   		\STATE $\mathcal{D} = \mathcal{D} \cup \tau_{\rm greedy}$
    	\STATE $C = C \cup \mathbf{\hat{N}}[\text{BEC}(\tau|\pi^*)]$

    \ENDWHILE
    \RETURN $\mathcal{D}$
\end{algorithmic}
\end{algorithm}

\subsection{Algorithm comparison}
To compare the performance of SCOT and UVM, 
we ran an experiment on random 9x9 grid worlds with 8-dimensional binary features per cell. We computed maximally informative demonstration sets with SCOT and UVM using trajectories consisting of single state-action pairs. We measured the performance loss for each algorithm by running IRL to find the maximum likelihood reward function given the demonstrations, and then calculating both the policy loss and the percentage of states where the resulting policy took a suboptimal action under the true reward.  Table~\ref{tab:randomGridComparison} shows that the UVM algorithm underestimates the size of the optimal teaching set of demonstrations, due to the difficulty of estimating volumes as discussed earlier, resulting in high performance loss. We tried sampling more points, but found that this only slightly improved performance loss while significantly increasing run-time. Compared to UVM, SCOT successfully finds demonstrations that lead to the correct policy, with orders of magnitude less computation. 


\begin{table*}[!t]
 \centering
  \begin{tabular}{ccccc}
    \toprule
        & Avg. number of $(s,a)$ pairs  & Avg. policy loss & Avg. \% incorrect actions & Avg. time (s) \\ 
    \midrule
    \midrule
UVM ($10^5$) & 5.150 & 1.539 & 31.420 & 567.961 \\
UVM ($10^6$) & 6.650 & 1.076 & 19.568 & 1620.578 \\
UVM ($10^7$) & 8.450 & 0.555 & 18.642 & 10291.365 \\
SCOT & 17.160 & 0.001 & 0.667 & 0.965 \\
    \bottomrule
  \end{tabular}
   \caption{Comparison of Uncertainty Volume Minimization (UVM) and Set Cover Optimal Teaching (SCOT) averaged across 20 random 9x9 grid worlds with 8-dimensional features. UVM($x$) was run using $x$ Monte Carlo samples. UVM underestimates the number of $(s,a)$ pairs needed to teach $\pi^*$.}
  \label{tab:randomGridComparison}
\end{table*}


To further explore the sensitivity of UVM to the number of features, we ran a test on a fixed size grid world with varying numbers of features. We used a deterministic teaching policy to ameliorate the problems with volume computation discussed in Section~\ref{sec:Cakmak}. We found that SCOT is robust for high-dimensional feature spaces, whereas UVM consistently underestimates the minimum number of demonstrations needed when there are 10 or more features, even when teaching a deterministic policy (see appendix for full details).

\section{Applications of Machine Teaching for IRL}
We now discuss some novel applications of machine teaching for IRL. 
One immediate application of SCOT is that it allows the first rigorous and efficiently computable definition of intrinsic teaching difficulty (teaching dimension) for IRL benchmarks. In the following section we demonstrate how SCOT can be used to benchmark active IRL algorithms by providing a lower bound on sample complexity. Finally, we demonstrate that SCOT can be incorporated into Bayesian IRL to allow more efficient use of informative demonstrations through counter-factual reasoning. 

\subsection{Bounding sample complexity for active IRL}
Our first application is to provide a lower bound on the sample complexity of learning a reward function via active queries \cite{lopes2009active,cuiactive2017,brown2018risk}.
To the best of our knowledge, no one has tried to benchmark existing algorithms against optimal queries, due to the combinatorial explosion of possible queries. Note that SCOT requires knowledge of the optimal policy, so it cannot be used directly as an active learning algorithm. Instead, we use SCOT as a tractable approximation to the optimal sequence of queries for active IRL. SCOT generates a sequence of maximally informative demonstrations via the set cover approximation. Thus, we can treat the sequence of demonstrations found by SCOT as an approximation of the best sequence of active queries to ask an oracle when performing active IRL.

We evaluated three active query strategies from the literature: \textit{Max Entropy}, a strategy proposed by Lopes et al. \cite{lopes2009active} that queries the state with the highest action entropy, \textit{Max Infogain}, a strategy proposed by Cui at al. \cite{cuiactive2017} that selects the trajectory with the largest expected change in the posterior $P(R|D)$, and \textit{Max VaR}, a recently proposed risk-aware active IRL strategy \cite{brown2018risk} that utilizes probabilistic performance bounds for IRL \cite{brown2018efficient} to query for the optimal action at the state where the maximum likelihood action given the current demonstrations has the highest 0.95-Value-at-Risk (95th-percentile policy loss over the posterior) \cite{jorion1997value}. We compare these algorithms against random queries and against the maximally informative sequence of queries found using SCOT.

We ran an experiment on 100 random 10x10 grid worlds with 10-dimensional binary features. 
Figure~\ref{fig:activeBenchmark0-1} shows the performance loss for each active IRL algorithm. Each iteration corresponds to a single state query and a corresponding optimal trajectory from that state. After adding each new trajectory to $\mathcal{D}$, the MAP reward function is found using Bayesian IRL \cite{ramachandran2007bayesian}, and the corresponding optimal policy is compared against the optimal policy under the true reward function. 

\begin{figure}
\centering
\includegraphics[scale=0.36]{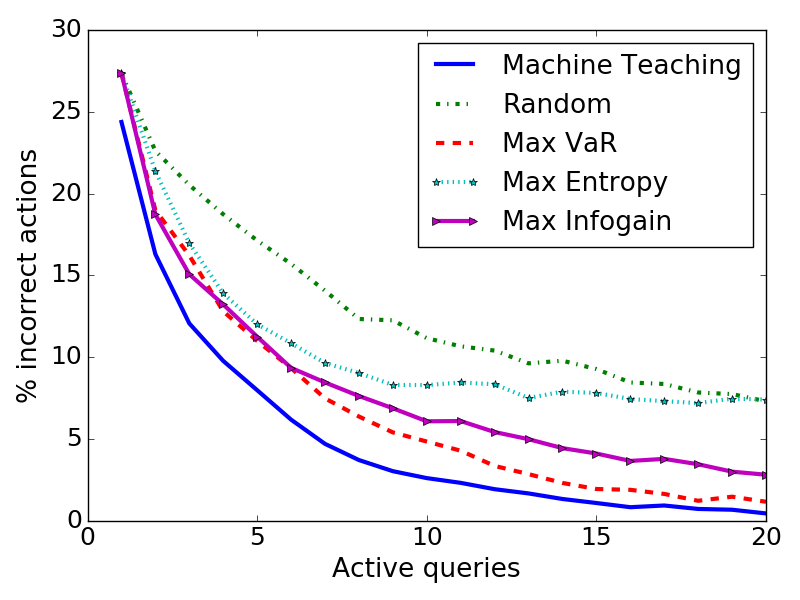}
\caption{Performance of active IRL algorithms compared to an approximately optimal machine teaching benchmark. Results are averaged over 100 random 10x10 grid worlds.}
\label{fig:activeBenchmark0-1}
\end{figure}

\begin{figure*}[t]
\centering
\includegraphics[scale=0.24]{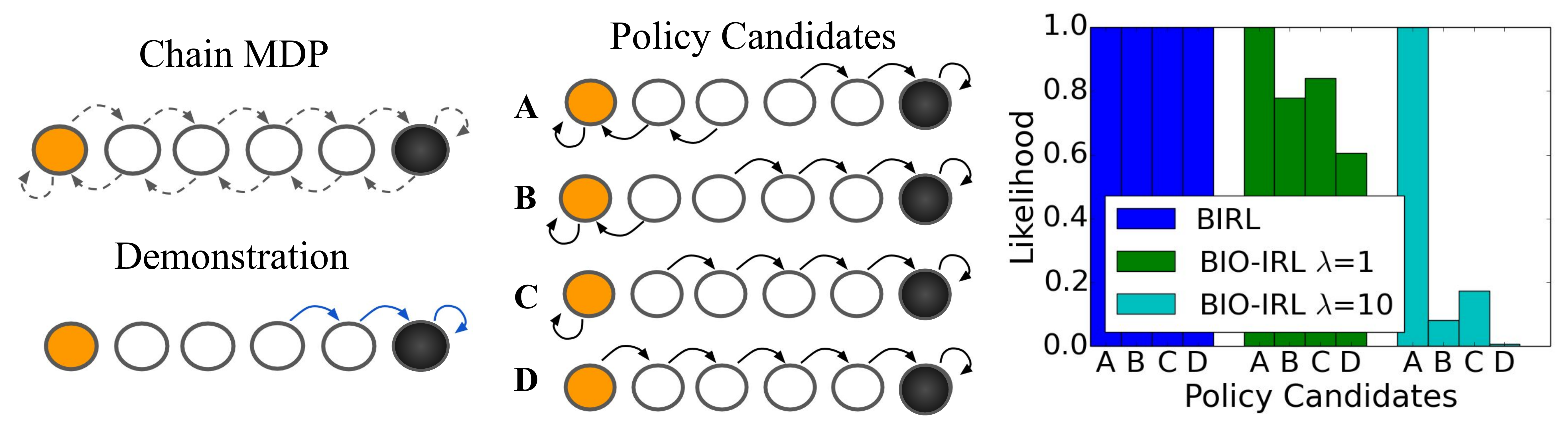}
\caption{Simple Markov chain with three features (orange, white, black) and two actions available in each state. Left: a single demonstrated trajectory. Center: all policies that are consistent with the demonstration. Right: likelihoods for BIRL and BIO-IRL with $\lambda$ = 1 and 10. BIRL gives all rewards that lead to any of the policy candidates a likelihood of 1.0 since it only reasons about the optimality of the demonstration under a hypothesis reward. BIO-IRL reasons about both the optimality of the demonstration and the informativeness of the demonstration and gives highest likelihood to reward functions that induce policy A.}
\label{fig:markovChain}
\end{figure*}

The results in Figure~\ref{fig:activeBenchmark0-1} show that all active IRL approaches perform similarly for early queries, but that Max Entropy ends up performing no better than random as the number of queries increases. This result matches the findings of prior work which showed that active entropy queries perform similarly to random queries for complex domains \cite{lopes2009active}. Max VaR and Max Infogain perform better than Max Entropy for later queries. By benchmarking the against SCOT we see that Max VaR queries are a good approximation of maximally informative queries. 

\subsection{Using optimal teaching to improve IRL}\label{subsec:BIO-IRL}
We next use machine teaching as a novel way to improve IRL when demonstrations are known to be informative. 
Human teachers are known to give highly informative, non i.i.d. demonstrations when teaching \cite{ho2016showing,shafto2008teaching}.  
For example, when giving demonstrations, human teachers do not randomly sample from the optimal policy, potentially giving the same (or highly similar) demonstration twice. However, existing IRL approaches usually assume demonstrations are i.i.d. \cite{ramachandran2007bayesian,ziebart2008maximum,babes2011apprenticeship,fu2017learning}. We propose an algorithm called Bayesian Information-Optimal IRL (BIO-IRL) that adds a notion of demonstrator informativeness to Bayesian IRL (BIRL) \cite{ramachandran2007bayesian}. Our insight is that if a learner knows it is receiving demonstrations from a teacher, then the learner should search for a reward function that makes the demonstrations look both optimal and informative. 


\subsubsection{BIO-IRL algorithm}
Our proposed algorithm leverages the assumption of an expert teacher: demonstrations not only follow $\pi^*$, but are also highly informative. 
We use the following likelihood for BIO-IRL:
\begin{equation}
P(\mathcal{D}|R) \propto P_{\rm info}(\mathcal{D} | R)\cdot \prod_{(s,a) \in \mathcal{D}} P((s,a) | R) 
\end{equation}
where $P((s,a) | R)$ is the standard BIRL softmax likelihood that computes the probability of taking action $a$ in state $s$ under $R$ and $P_{\rm info}(\mathcal{D} | R)$ measures how informative the entire demonstration set $\mathcal{D}$ appears under $R$.

We compute $P_{\rm info}(\mathcal{D} | R)$ as follows. Given a demonstration set $\mathcal{D}$ and a hypothesis reward function $R$, we first compute the information gap, infoGap$(\mathcal{D}, R)$, which uses behavioral equivalence classes to compare the relative informativeness of $\mathcal{D}$ under $R$, with the informativeness of the maximally informative teaching set $\mathcal{D}^*$ under $R$ (see  Algorithm~\ref{alg:infoGap}). We estimate informativeness by computing the angular similarity between half-space normal vectors (see appendix for details). 
BIO-IRL uses the absolute difference between angular similarities to counterfactually reason about the gap in informativeness between the actual demonstrations and an equally sized set of demonstrations designed to teach $R$.
Given the actual demonstration $\mathcal{D}$ and the machine teaching demonstration $\mathcal{D}^*$ under $R$, we let 
\begin{equation}
P_{\rm info}(\mathcal{D} | R) \propto \exp(- \lambda \cdot \text{infoGap}(\mathcal{D}, R))
\end{equation}
where $\lambda\geq0$ is a hyperparameter modeling the confidence that the demonstrations are informative. If $\lambda = 0$, then BIO-IRL is equivalent to standard BIRL. 

The main computational bottlenecks in Algorithm~\ref{alg:infoGap} are finding the optimal policy for $R^*$, and computing the expected feature counts for state-action pairs that are used to calculate the various BEC constraints. 
Computing the optimal policy $\pi^*$ is already required for BIRL. Computing the expected feature counts for the behavioral equivalence classes is equivalent to performing a policy evaluation step which is computationally cheaper than fully solving for $\pi^*$ \cite{barreto2017successor}. By caching BEC$(\pi^*)$ for each reward function it is possible to save significant computation time during MCMC by reusing the cached BEC if a proposal $R$ satisfies the vectorized version of Theorem 1 (see \cite{choi2011map} and Corollary~\ref{thm:ngRussell} in the appendix).


\begin{algorithm}[t] 
\caption{infoGap$(\mathcal{D}, R)$} 
\label{alg:infoGap} 
\begin{algorithmic}[1] 
	\STATE Calculate $\pi^*$ under $R$.
    \STATE Calculate BEC$(\pi^*)$, BEC$(\mathcal{D} | \pi^*)$
    \STATE $\mathcal{D}^* \gets $ \textbf{SCOT}($\pi^*$)
    \STATE $m \gets$ number of trajectories in $\mathcal{D}$
    \STATE $\mathcal{D}^*_{1:m} \gets$ first $m$ trajectories in $\mathcal{D}^*$
    \STATE Calculate BEC$(\bar{\mathcal{D}}^* | \pi^*)$
    \STATE infoDemo $\gets$ \textbf{angSim}$(\mathbf{\hat{N}}[\text{BEC}(\mathcal{D}|\pi^*)],\mathbf{\hat{N}}[\text{BEC}(\pi^*)])$
    \STATE infoOpt $\gets$ \textbf{angSim}$(\mathbf{\hat{N}}[\text{BEC}(\mathcal{D}^*_{1:m}|\pi^*)], \mathbf{\hat{N}}[\text{BEC}(\pi^*)])$
    \RETURN \textbf{abs}(infoDemo - infoOpt)
\end{algorithmic}
\end{algorithm}

\subsubsection{Experiments}
Consider the Markov chain in Figure~\ref{fig:markovChain}. If an information-optimal demonstrator gives the single demonstration shown on the left, then BIO-IRL assumes that both the orange and black features are equally preferable over white. This is because, given different preferences, an informative teacher would have demonstrated a different trajectory. For example, if the black feature was always preferred over white and orange, a maximally informative demonstration would have demonstrated a trajectory that started in the leftmost state and moved right until it reached the rightmost state. 

Because the likelihood function for standard BIRL only measures how well the demonstration matches the optimal policy for a given reward function, BIRL assigns equal likelihoods to all reward functions that result in one of the policy candidates shown in the center of the Figure~\ref{fig:markovChain}. The bar graph in Figure~\ref{fig:markovChain} shows the likelihood of each policy candidate under BIRL and BIO-IRL with different $\lambda$ parameters. Rather than assigning equal likelihoods, BIO-IRL puts higher likelihood on Policy A, the policy that makes the demonstration appear both optimal and informative. Changing the likelihood function to reflect informative demonstrations results in a tighter posterior distribution, which is beneficial when reasoning about safety in IRL \cite{brown2018efficient}.

We also evaluated BIO-IRL in a ball sorting task shown in Figure~\ref{subfig:table}. In this task, balls start in one of 25 evenly spaced starting conditions and need to be moved into one of four bins located on the corners of the table. Demonstrations are given by selecting an initial state for the ball and moving the ball until it is in one of the bins. Actions are discretized to the four cardinal directions along the table top and the reward is a linear combination of five indicator features representing whether the ball is in one of the four bins or on the table. We generated demonstrations from 50 random rewards with $\gamma = 0.95$. This provides a wide variety of preferences over bin placement depending on the initial distance of the ball to the different bins. We used SCOT to generate informative demonstrations which were given sequentially to BIRL and BIO-IRL. Figure~\ref{fig:BIO_exp} shows that BIO-IRL is able leverage informative, non-i.i.d. demonstrations to learn more efficiently than BIRL (see appendix for details).

\begin{figure*}[t]
\centering
\subfigure[]{
\includegraphics[scale=0.1]{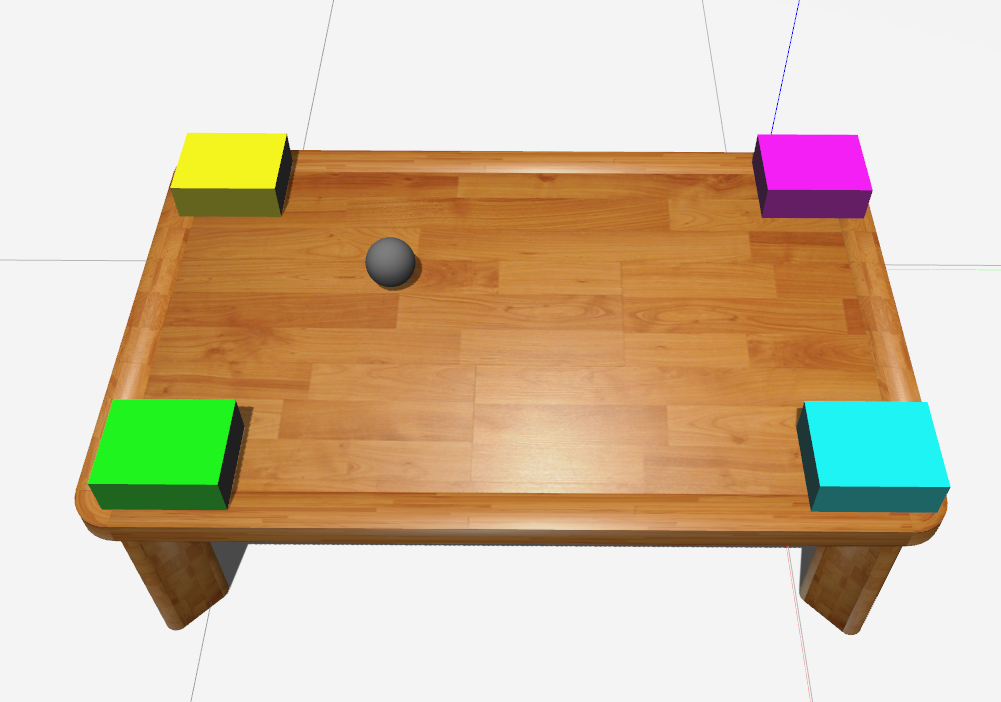}
\label{subfig:table}
}
\subfigure[]{
\includegraphics[scale=0.28]{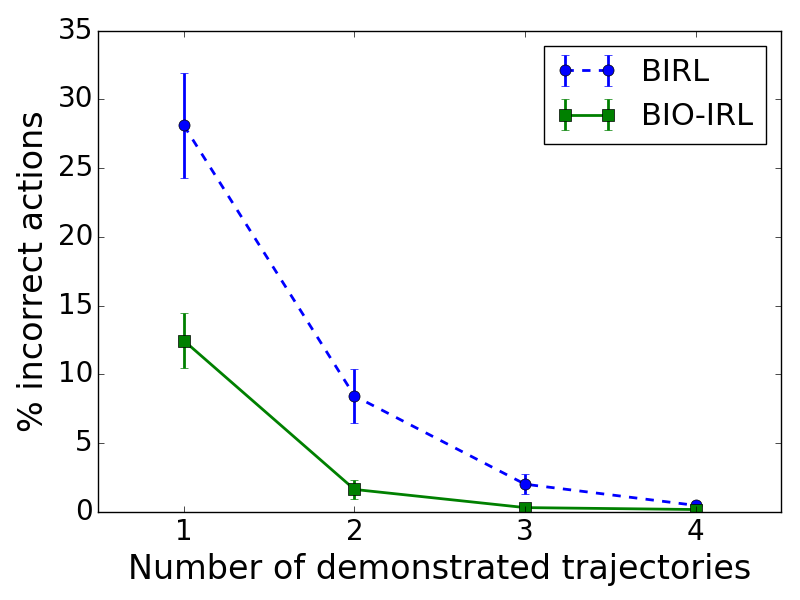}
\label{subfig:fourGoals_aloss}
}
\subfigure[]{
\includegraphics[scale=0.28]{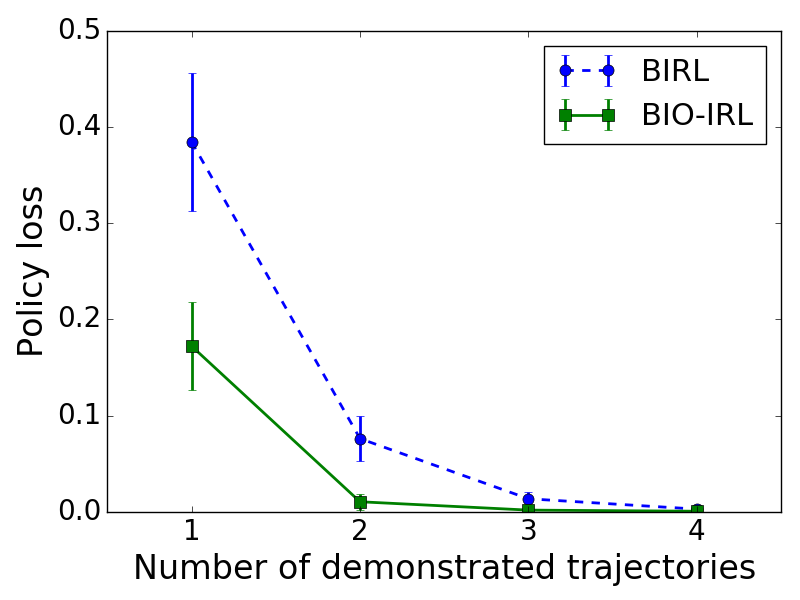}
\label{subfig:fourGoals_ploss}
}

\caption{(a) Ball sorting task. The ball starts at one of 36 positions on the table and must be moved into one of four bins depending on the demonstrator's preferences. 0-1 action losses (b) and policy value losses for MAP policies found by BIO-IRL and BIRL when receiving informative demonstrations. Error bars are 95\% confidence intervals around the mean from 50 trials.}
\label{fig:BIO_exp}
\end{figure*}


\section{Summary and Future Work}
We formalized the problem of optimally teaching an IRL agent as a machine teaching problem and proposed an efficient approximation algorithm, SCOT, to solve the machine teaching problem for IRL. Through a set-cover reduction we avoid sampling and use submodularity to achieve an efficient approximation algorithm with theoretical guarantees that the learned reward and policy are correct. Our proposed approach enables an efficient and robust algorithm for selecting maximally informative demonstrations that shows several orders of magnitude improvement in computation time over prior work and scales better to higher-dimensional problems. 


For our first application of machine teaching for IRL we examined using SCOT to approximate lower bounds on sample complexity for active IRL algorithms. Benchmarking active IRL against an approximately optimal query strategy shows that a recent risk-sensitive IRL approach \cite{brown2018risk} is approaching the machine teaching lower bound on sample complexity for grid navigation tasks.
 
Our second application of machine teaching demonstrated that an agent that knows it is receiving informative demonstrations can learn more efficiently than a standard Bayesian IRL approach. When humans teach each other we typically do not randomly pick examples to show, rather expert teachers cherry-pick pick highly informative demonstrations that highlight important details and guide the learner away from common pitfalls. However, most IRL algorithms assume that demonstrations are sampled i.i.d. from the demonstrator's policy. We proposed BIO-IRL as a way to use machine teaching to relax i.i.d. assumptions and correct for demonstration bias when learning from an informative teacher.  

One area of future work is applying SCOT to continuous state-spaces. Expected feature counts can be efficiently computed even if the model is unknown and the state-space is continuous \cite{barreto2017successor}. Additionally, behavioral equivalence classes can be approximated for continuous state spaces by sampling a representative set of starting states. Given an approximation of the BEC for continuous spaces, SCOT is still guaranteed to terminate and retain efficiency (see appendix for details). Future work also includes using SCOT to approximate the theoretical lower bound on sample complexity for more complicated active learning domains, benchmarking other active learning approaches, and using our proposed machine teaching framework to rank common LfD benchmarks according to their inherent teaching difficulty.
Future work should also investigate methods for estimating demonstrator informativeness as well as methods to detect and correct for other demonstrator biases in order to learn more efficiently from non-i.i.d. demonstrations. 
 
There are many other interesting possible applications of machine teaching to IRL. One potential application is to use machine teaching to study the robustness of IRL algorithms to poor or malicious demonstrations by studying optimal demonstration set attacks and defenses \cite{mei2015using,alfeld2017explicit}. Machine teaching for sequential decision making tasks also has applications in ad hoc teamwork \cite{stone2010ad} and explainable AI \cite{gunning2017explainable}. In ad hoc teamwork, one or more robots or agents may need to efficiently provide information about their intentions without having a reliable or agreed upon communication protocol. Agents could use our proposed machine teaching algorithm to devise maximally informative trajectories to convey intent to other agents. Similarly, in the context of explainable AI, a robot or machine may need to convey its intention or objectives to a human \cite{huang2017enabling}. Simply showing a few maximally informative examples can be a simple yet powerful way to convey intention.


\section*{ Acknowledgments}
This work has taken place in the Personal Autonomous Robotics Lab (PeARL) at The University of Texas at Austin. PeARL research is supported in part by the NSF (IIS-1724157, IIS-1638107, IIS-1617639, IIS-1749204) and ONR (N00014-18-2243).

\bibliographystyle{aaai}

\bibliography{aaai19bib}

\appendix

\section{Behavioral equivalence classes}

\begin{customthm}{1} \label{thm:contFeasible} \cite{ng2000algorithms}
Given an MDP, BEC($\pi$) is given by the following intersection of halfspaces:
\begin{eqnarray}
\mathbf{w}^T (\mu_{\pi}^{(s,a)} - \mu_{\pi}^{(s,b)}) \geq 0, \\ \forall a \in \arg\max_{a'\in \mathcal{A}} Q^*(s,a'), b \in \mathcal{A}, s \in \mathcal{S}
\end{eqnarray}
$\mathbf{w} \in \mathbb{R}^k$ are the reward function weights, $\mu_{\pi}^{(s,a)} = \mathbb{E}[\sum_{t=0}^{\infty} \gamma^t \phi(s_t) | \pi, s_0 = s, a_0 = a]$, is the vector of expected feature counts from taking action $a$ in state $s$ and acting optimally thereafter. 
\end{customthm}
 \begin{proof}
 In every state $s$ we can assume that there is one or more optimal actions $a$. For each optimal action $a \in \arg\max_{a'\in \mathcal{A}} Q^*(s,a')$, we then have by definition that
 \begin{equation}
 Q^*(s,a) \geq Q^*(s,b), \; \forall b \in A
 \end{equation}
 Rewriting this in terms of expected discounted feature counts we have
 \begin{equation}
 w^T \mu_{\pi}^{(s,a)} \geq w^T \mu_{\pi}^{(s,b)}, \; \forall b \in A
 \end{equation}
 Thus, the behavioral equivalence class is the intersection of the following half-spaces 
 \begin{eqnarray}
 w^T (\mu_{\pi}^{(s,a)} -  \mu_{\pi}^{(s,b)}) \geq 0, \\ \forall a \in \arg\max_{a'\in \mathcal{A}} Q^*(s,a'), b \in \mathcal{A}, s \in \mathcal{S}.
 \end{eqnarray}
 \end{proof}

We can define the BEC for a set of demonstrations $\mathcal{D}$ from a policy $\pi$ similarly:

\begin{customcor}{1}\label{cor:feasibleDemo}
BEC$(\mathcal{D} | \pi)$, is given by the following intersection of halfspaces:
\begin{equation}
\mathbf{w}^T (\mu_{\pi}(s,a) - \mu_{\pi}(s,b)) \geq 0,\; \forall (s,a) \in \mathcal{D}, b \in \mathcal{A}.  
\end{equation}
\end{customcor}
 \begin{proof}
 The proof follows from the proof of Theorem~\ref{thm:contFeasible} by only considering half-spaces corresponding to optimal $(s,a)$ pairs in the demonstration.
 \end{proof}

\section{Example}

Given an MDP with finite states and actions, we can calculate BEC($\pi$) via the following result, proved by Ng and Russell \cite{ng2000algorithms}, which is equivalent to Theorem~1.

\begin{corollary}\label{thm:ngRussell} \cite{ng2000algorithms}
Given a finite state space $\mathcal{S}$ with a finite number of actions $\mathcal{A}$, policy $\pi$ is optimal if and only if reward function $R$ satisfies
\begin{equation}
(\mathbf{T_\pi} - \mathbf{T_a})(\mathbf{I} - \gamma \mathbf{T_\pi})^{-1} \mathbf{R} \geq 0, \; \forall a \in \mathcal{A}
\end{equation}
where $\mathbf{T_a}$ is the transition matrix associated with always taking action $a$, $\mathbf{T_{\pi}}$ is the transition matrix associated with policy $\pi$, and $\mathbf{R}$ is the column vector of rewards for each state $s \in \mathcal{S}$.
\end{corollary}
\begin{proof}
See \cite{ng2000algorithms}.
\end{proof}

Consider the grid world shown in Figure~\ref{subfig:simpleNg} (see the main text) with four actions (up, down, left, right) available in each state and deterministic transitions. Actions that would leave the grid boundary (such as taking the up action from the states in the top row) result in a self-transition. We computed the BEC region defined by Theorem~\ref{thm:ngRussell}: $$(\mathbf{T_\pi} - \mathbf{T_a})(\mathbf{I} - \gamma \mathbf{T_\pi})^{-1} \mathbf{\Phi} \mathbf{w} \geq 0$$ for  $a \in \{up, down, left, right \}$, setting $\gamma = 0.9$ and using a featurized reward function $R(s) = w^T \phi(s)$, where $w = (w_0, w_1)$ is the feature weight vector with $w_0$ indicating the reward weight for a ``white" cell and $w_1$ indicating the reward weight for a ``shaded" cell. 
We can express the vector of state rewards as $\mathbf{R} = \Phi \mathbf{w}$, where $$\Phi =
[
\phi(s_0)^T,
\phi(s_1)^T,
\phi(s_2)^T,
\phi(s_3)^T,
\phi(s_4)^T,
\phi(s_5)^T]
$$
and $\phi(s_i) = (1,0)$ for $i \in \{0,2,3,4,5\}$ and $\phi(s_1) = (0,1)$, are the feature vectors for each state numbered left to right top to bottom.


The computation results in the following non-redundant constraints that fully define BEC$(\pi)$ for $\pi$ given in Figure~\ref{fig:exampleFeasiblePolicy}:
\begin{align}
2.539 w_0 - w_1\geq 0, \;\;
-w_0 \geq 0.
\end{align}
These constraints exactly describe the set of rewards that make the policy shown in Figure~\ref{subfig:simpleNg} optimal. 
 This can be seen by noting that the constraints ensure that all feature weights are non-positive, because a positive weights would cause the optimal policy to avoid early termination to accumulate as much reward as possible. We also have the constraint that if we start in state 3, it is better to move down and around the shaded state then to go directly to the terminal state, this means
 \begin{align}
 w_0 + \gamma w_0 + \gamma^2 w_0 + \gamma^3 w_0 + \gamma^4 w_0 \geq w_0 + \gamma w_1 + \gamma^2 w_0\\
 \Leftrightarrow (1 + \gamma^2 + \gamma^3) w_0 \geq w_1
 \end{align}
 which gives us the second constraint using $\gamma = 0.9$. It is straightforward to complete similar inequalities for all states to check that 
 $0\geq w_0$ and $2.539w_0 \geq w_1$ are the only non-redundant constraints.

Computing the intersection of halfspaces corresponding to the demonstration gives the following convex cone
\begin{align}
-w_1 \geq 0, \;\;
w_1 - w_2\geq 0.
\end{align}
Note that the second constraint on the difference between the two feature weights is looser than the BEC region for the entire optimal policy. This is because the demonstration only shows that both features are non-positive (making the terminal a goal) and that $w_2$ is no better than $w_1$ (otherwise the demonstration would have gone through the shaded region). The demonstration leaves open the possibility that all feature weights are equal. We also note that if the demonstration had started in the top right cell, the BEC region of the demonstration would equal the BEC region of the optimal policy. 

\section{Uncertainty Volume Minimization Algorithm}
Pseudo-code for the UVM algorithm is shown in Algorithm~\ref{alg:UVM}.

Consider the MDP shown in Figure~\ref{fig:zeroUncertaintyExample}. When the UVM algorithm is run on this task the algorithm exits the while-loop reporting that the uncertainty has gone to zero and returns the  demonstration set shown in Figure~\ref{fig:zeroUncertaintyDemos}. This is clearly not an optimal demonstration since the starting state in the upper right has never been demonstrated so the agent has no idea what it should do from that state.

\begin{figure}[]
\centering
\includegraphics[scale=0.2]{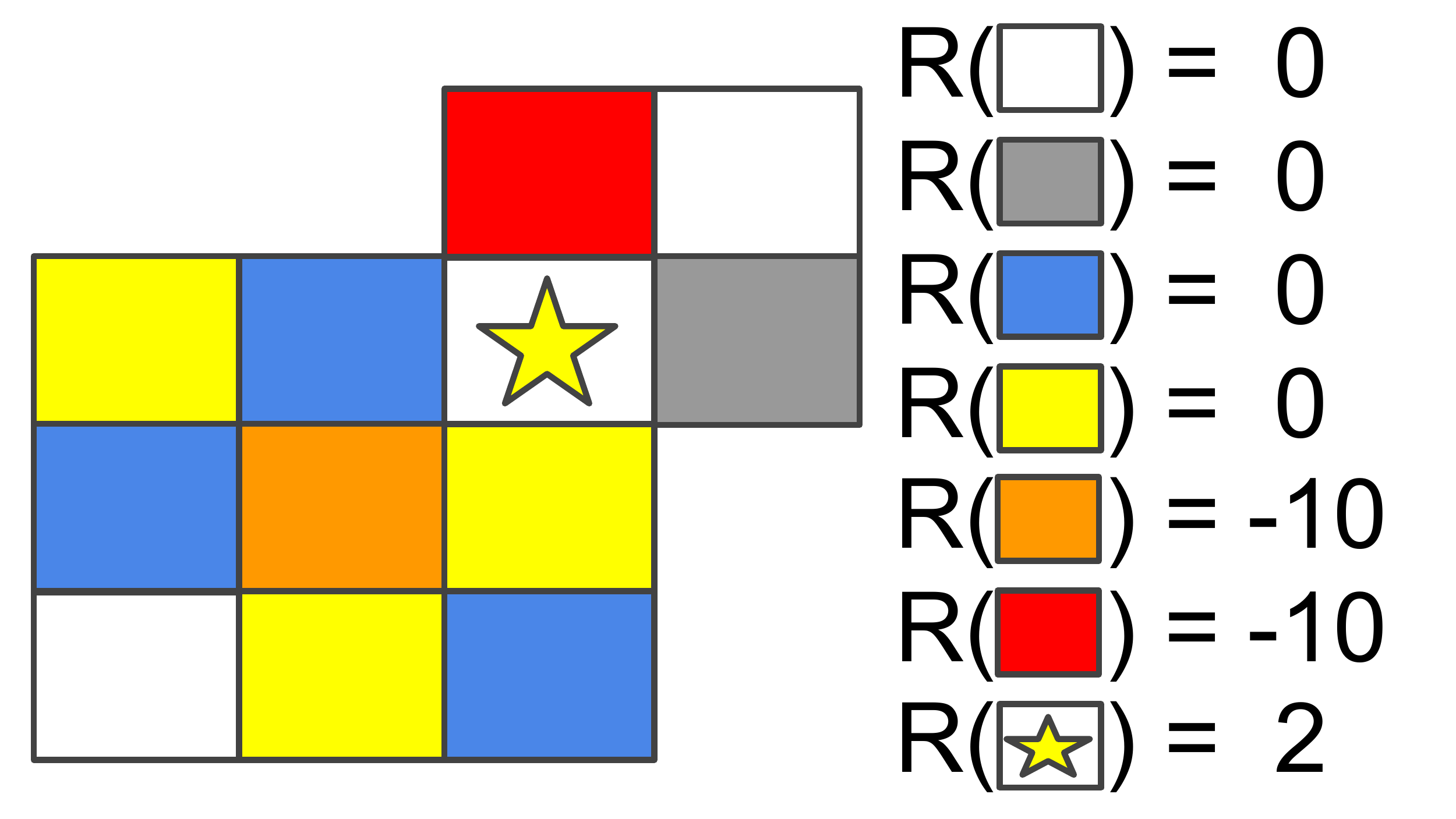}
\caption{Grid MDP with actions up, down, left, and right. The cell with a star is a terminal state. All other states are possible starting states.}
\label{fig:zeroUncertaintyExample}
\end{figure}

\begin{figure}[]
\centering
\includegraphics[scale=0.2]{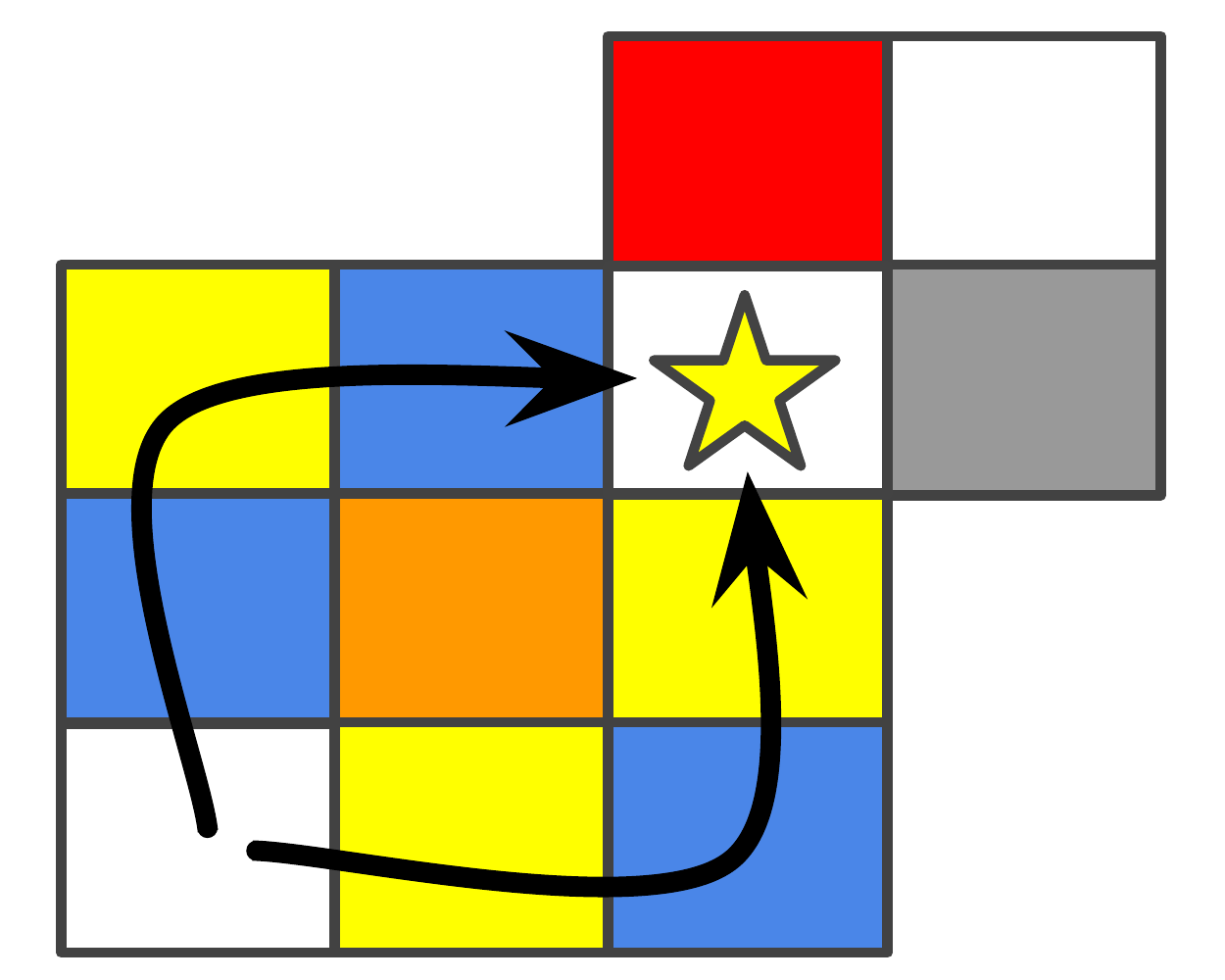}
\caption{Demonstrations found by the UVM algorithm \cite{cakmak2012algorithmic} that result in a false conclusion of zero uncertainty over the demonstrator's true reward. This false certainty comes despite not receiving any evidence about the relative rewards of the red and gray states.}
\label{fig:zeroUncertaintyDemos}
\end{figure}

This highlights a problem that the UVM algorithm has with estimating volumes. 
If there are ever two actions that are optimal in a given state $s$, and those two actions (call them $a$ and $b$) are demonstrated, then we will have the following halfspace constraints:
\begin{eqnarray}
w^T(\mu^{(s,a)}_\pi - \mu^{(s,b)}_\pi) \geq 0 \\
w^T(\mu^{(s,b)}_\pi - \mu^{(s,a)}_\pi) \geq 0
\end{eqnarray}
Thus, we have $w^T(\mu^{(s,b)}_\pi - \mu^{(s,a)}_\pi) = 0$. 

This is problematic since any strict subspace of $\mathbb{R}^k$ has measure zero, resulting in an uncertainty volume of zero. Thus, for any optimal policy where there exists a state with two or more optimal actions, the UVM algorithm will terminate with zero uncertainty if these two optimal actions are added to $\mathcal{D}$. This is true, even if this leaves an entire $(k-1)$ dimensional subspace of uncertainty over the reward. 

 \begin{algorithm} 
 \caption{Uncertainty Volume Minimization} 
 \label{alg:UVM} 
 \begin{algorithmic}[1] 
     \REQUIRE Set of possible initial states $S_0$
     \REQUIRE Feature weights $\mathbf{w}$ of the optimal reward function    
     \STATE Initialize $\mathcal{D} \gets \emptyset$
     \STATE Compute optimal policy $\pi^*$ based on $\mathbf{w}$
     \REPEAT
     		\STATE $\zeta_{\rm best} \gets \textbf{null}$
 		\FORALL{$s_0 \in S_0$}
 			\STATE Generate $K$ trajectories from $s_0$ following $\pi^*$
 			\FOR{$j \in [1,K]$}
 				\IF{$G(\mathcal{D}\cup \zeta_j) < G(\mathcal{D} \cup \zeta_{\rm best}) $ and $\zeta_j \notin \mathcal{D}$}
 				\STATE $\zeta_{\rm best} \gets \zeta_j$
 				\ENDIF
 			\ENDFOR
 		\ENDFOR
 		\STATE $\mathcal{D} \gets \mathcal{D} \cup \zeta_{\rm best}$
     \UNTIL{$\zeta_{\rm best} \textbf{  is  } \textbf{null}$}
     \RETURN Demonstration set $\mathcal{D}$
 \end{algorithmic}
 \end{algorithm}

 \section{Removing redundant constraints}
 A redundant constraint is one that can be removed without changing the BEC region.
 We can find redundant constraints efficiently using linear programming. To check if a constraint $a^Tx \leq b$ is binding we can remove that constraint and solve the linear program with $\max_x a^Tx$ as the objective. If the optimal solution is still constrained to be less than or equal to $b$ even when the constraint is removed, then the constraint can be removed. However, if the optimal value is greater than $b$ then the constraint is non-redundant. Thus, all redundant constraints can be removed by making one pass through the constraints, where each constraint is immediately removed if redundant.
 We optimize this approach by first normalizing each constraint and removing duplicates and any trivial, all zero constraints.


%

\section{Set-cover algorithm termination}
\begin{customprop}{2}
The set-cover machine teaching algorithm for IRL always terminates.
\end{customprop}
\begin{proof}
To prove that our algorithm always terminates, consider the polyhedral cone $C_f = \{x \in \mathbb{R}^n \mid A_fx \geq 0  \}$ that represents BEC$(\pi^*)$. Each demonstrated state action pair, $(s,a)$ defines a set of half spaces
\begin{equation}
w^T (\mu(s,\pi^*(s)) -  \mu(s,a)) \geq 0, \; \forall a \in \mathcal{A}.
\end{equation}

When the intersection of the halfspaces for every $(s,a) \in \mathcal{D}$ is equal to BEC($\pi^*$), the algorithm terminates an returns $\mathcal{D}$. For discrete domains, BEC$(\pi^*)$ is simply the intersection of a finite number of half-spaces from every optimal $(s,a)$ pair, so once every optimal $(s,a)$ has been demonstrated, the algorithm is guaranteed to terminate. In practice, BEC$(\pi^*)$ can be fully defined by only a subset of the possible demonstrations. Thus, our machine teaching algorithm seeks to select the minimum number of demonstrations that cover all of the rows of $A_f$. 

For continuous domains, we cannot fully enumerate every optimal $(s,a)$-pair. However, it is possible to approximate BEC$(\pi^*)$ by sampling optimal rollouts from the state space.
We then solve the constraint set-cover problem using these same sampled rollouts, so we are again guaranteed to terminate once all demonstrations are chosen, and will likely terminate after only selecting a small subset of the sampled demonstrations.
\end{proof}

\begin{proposition} \label{prop:submodular}
The set-cover machine teaching algorithm is a $(1-1/e)$-approximation to the minimum number of demonstrations needed to fully define BEC$(\pi^*)$.
\end{proposition}
\begin{proof}
This result follows from the submodularity of the set cover problem \cite{wolsey1982analysis,nemhauser1978analysis}.
\end{proof}

\section{Optimality of set cover algorithm}
We now prove the condition under which our proposed algorithm is a (1-1/e)-approximation of the solution to the Machine Teaching Problem for IRL. 

Both the UVM and SCOT algorithms focus on teaching halfspaces to an IRL algorithm to define the behavioral equivalence region, BEC$(\pi^*)$. Thus, they assume that when the IRL algorithm receives state-action pair $(s,a)$ from the demonstrator, the IRL algorithm will enforce the constraint that $Q^*(s,a) \geq Q^*(s,b)$, $\forall b \in \mathcal{A}$. We call this assumption the \textit{halfspace assumption}.

\begin{definition}
The halfspace assumption is that $Q^*(s,a) \geq Q^*(s,b)$, $\forall b \in \mathcal{A}, (s,a) \in \mathcal{D}$.
\end{definition}

We now prove that, under the assumption of error-free demonstrations, three common IRL algorithms make the halfspace assumption: Bayesian IRL \cite{ramachandran2007bayesian,choi2011map}, Policy Matching \cite{neu2007apprenticeship}, and Maximum Likelihood IRL \cite{lopes2009active,babes2011apprenticeship}.

\begin{lemma} \label{lemma:halfspace_birl}
Under the assumption of error-free demonstrations, Bayesian IRL \cite{ramachandran2007bayesian,choi2011map} makes the halfspace assumption.
\end{lemma}
\begin{proof}
Bayesian IRL uses likelihood 
\begin{equation} \label{eqn:birl_likelihood}
P_{\rm opt}(\mathcal{D} | R) = \prod_{(s,a) \in \mathcal{D}}\frac{e^{\alpha Q^*(s,a)}}{ \sum_{b \in A} e^{\alpha Q^*(s,b)}}
\end{equation}
where $Q^*$ is the optimal $Q$-function under reward function $R$ and $\alpha\in [0,\infty)$ represents the confidence that the demonstrations come from $\pi^*$. As $\alpha \rightarrow \infty$, Bayesian IRL assume error-free demonstrations and we have 
\begin{equation} \label{eqn:birl_likelihood}
\lim_{\alpha \rightarrow \infty} P_{\rm opt}(\mathcal{D} | R) = 0 \iff  \exists b \in \mathcal{A}, \text{ s.t. } Q^*(s,a) < Q^*(s,b) .
\end{equation}
Thus, Bayesian IRL only gives positive likelihood to reward functions $R$, if $Q^*(s,a) \geq Q^*(s,b)$ $\forall b \in \mathcal{A}, (s,a) \in \mathcal{D}$.
\end{proof}

\begin{corollary} \label{cor:halfspace_policymatching}
Under the assumption of error-free demonstrations, Policy Matching \cite{neu2007apprenticeship} makes the optimal teaching assumption.
\end{corollary}
\begin{proof}
Melo et al. \cite{melo2010analysis} proved that Bayesian IRL \cite{ramachandran2007bayesian} and Policy Matching \cite{neu2007apprenticeship} share the same reward solution space. Thus, the lemma follows from the previous proof.
\end{proof}

\begin{corollary}\label{cor:halfspace_mlirl}
Under the assumption of error-free demonstrations, Maximum Likelihood IRL \cite{lopes2009active,babes2011apprenticeship} makes the optimal teaching assumption.
\end{corollary}
\begin{proof}
Maximum Likelihood IRL uses the same likelihood function as Bayesian IRL, thus the result follows from the previous lemma.
\end{proof}

We can now prove the following Theorem:

\begin{customthm}{2}
Under the assumption of error-free demonstrations, SCOT is a $(1-1/e)$-approximation to the Machine Teaching Problem for IRL (Section~\ref{subsec:problemdef}) for the following learning algorithms:
\begin{itemize}
\item Bayesian Inverse Reinforcement Learning \cite{ramachandran2007bayesian,choi2011map}
\item Policy Matching \cite{neu2007apprenticeship}
\item Maximum Likelihood Inverse Reinforcement Learning \cite{lopes2009active,babes2011apprenticeship}
\end{itemize}
\end{customthm}
\begin{proof}
Given an IRL algorithm that makes the halfspace assumption, by Proposition~\ref{prop:submodular} SCOT will find a set of demonstrations that are a (1-1/e)-approximation of the maximally informative demonstration set. Thus, by Lemma~\ref{lemma:halfspace_birl}, in the limit as $\alpha \rightarrow \infty$, the SCOT machine teaching algorithm is a (1-1/e)-approximation to the optimal demonstration set for Bayesian IRL. Similarly, by corollaries~\ref{cor:halfspace_policymatching} and \ref{cor:halfspace_mlirl}, SCOT is a (1-1/e)-approximation to the optimal demonstration set for Policy Matching and Maximum Likelihood IRL.
\end{proof}

\begin{table*}[!th]
   \caption{Comparison of different optimal teaching algorithms across random 9x9 grid worlds with 8-dimensional binary features. Algorithms compared are Uncertainty Volume Minimization (UVM), Set Cover Optimal Teaching (SCOT) with and without redundancies, and random sampling from the optimal policy. UVM($x$) was run using $x$ Monte Carlo samples for volume estimation. Results show the average number of state-action pairs in the demonstration set $\mathcal{D}$, the average number of suboptimal actions when performing IRL using $\mathcal{D}$ learned policy compared to optimal, and the average run time of the optimal teaching algorithm in seconds. All results are averaged over 20 replicates.}
   \label{tab:uvm_vs_setcover}
 \centering
  \begin{tabular}{lcccccc}
    \toprule
        & Ave. $(s,a)$ pairs & Ave. policy loss & Ave. \% incorrect actions & Ave. time (s) \\ 
    \midrule
    \midrule
UVM ($10^4$) & 3.850 & 1.722 & 44.074 & 247.944 \\
UVM ($10^5$) & 5.150 & 1.539 & 31.420 & 567.961 \\
UVM ($10^6$) & 6.650 & 1.076 & 19.568 & 1620.578 \\
UVM ($10^7$) & 8.450 & 0.555 & 18.642 & 10291.365 \\
SCOT (redundant) & 66.740 & 0.001 & 0.617 & 12.407 \\
SCOT & 17.160 & 0.001 & 0.667 & 0.965 \\
Random & 17.700 & 0.015 & 10.123 & 0.000 \\

    \bottomrule
  \end{tabular}

  \label{tab:randomGridComparison}
\end{table*}

\section{Algorithm comparison full results}
We compared the SCOT algorithm with the UVM algorithm of Cakmak and Lopes \cite{cakmak2012algorithmic}. We also report here a comparison against SCOT without removing redundancies and against random selected demonstrations from the optimal policy. 

We ran an experiment on random 9x9 grid worlds with 8 binary indicator features per cell with one feature active per cell and $\gamma = 0.95$. For this experiment the demonstrations were single state-action pairs. We measured the 0-1 policy loss \cite{michini2015bayesian} for each demonstration set by computing the percentage of states where the resulting policy took a suboptimal action under the true reward. The policy was found by first finding the maximum likelihood reward function \cite{lopes2009active,babes2011apprenticeship}, by using BIRL \cite{ramachandran2007bayesian} with a uniform prior and $\alpha=100$. We ran the MCMC chain for 10,000 steps using $\alpha = 100$ and step size of $0.005$.  Given the maximum likelihood reward function, the corresponding policy was then found using value iteration. The results are shown in Table~\ref{tab:randomGridComparison}. 

We found that the UVM algorithm usually underestimates the size of the optimal teaching set of demonstrations, due to the difficulty of estimating volumes as discussed earlier, resulting in high 0-1 loss. We tried sampling more points, but found that this only slightly improved 0-1 loss while significantly increasing run-time. Compared to UVM, our results show that SCOT is more accurate and more efficient---it can successfully find good demonstrations that lead to IRL learning the correct reward and corresponding policy, with orders of magnitude less computation time. We also found that removing redundant halfspaces is important; keeping redundant constraints results in slightly lower average performance loss, due to the randomness in MCMC, but finds demonstration sets that are much larger. SCOT removes redundant half-space constraints which results in solutions with fewer state-action pairs and a faster run-time, since the set-cover problem is substantially reduced in size.


We also compared against randomly sampling 20 state-action pairs from the optimal policy. The results show that SCOT is able to find informative demonstrations that significantly reduce the number of $(s,a)$ pairs to teach a policy, compared to sampling i.i.d. from the policy.

To further explore the sensitivity of UVM to the number of features, we ran a test on a fixed 6x6 size grid world with varying numbers of features. We wanted to see if using UVM and SCOT to teach a deterministic policy would avoid the early stopping problem for UVM by avoiding the problem of having multiple optimal actions from the same starting position. We also investigated using longer demonstrations with a horizon of 6. We found that the SCOT algorithm picks more demonstrations as the number of features increases, however, the UVM algorithm cannot reliably estimate volumes for higher-dimensional spaces. The results are shown in Figure~\ref{fig:uvm_vs_scot}. The number of state-action pairs in the demonstrations are shown to plateau and even slightly decrease for UVM. Thus, UVM underestimates the number of required demonstrations to teach an optimal policy for high-dimensional features while still requiring nearly than three orders of magnitude more computation time.

\begin{figure}[!th]
\centering
\subfigure[]{
\includegraphics[scale=0.3]{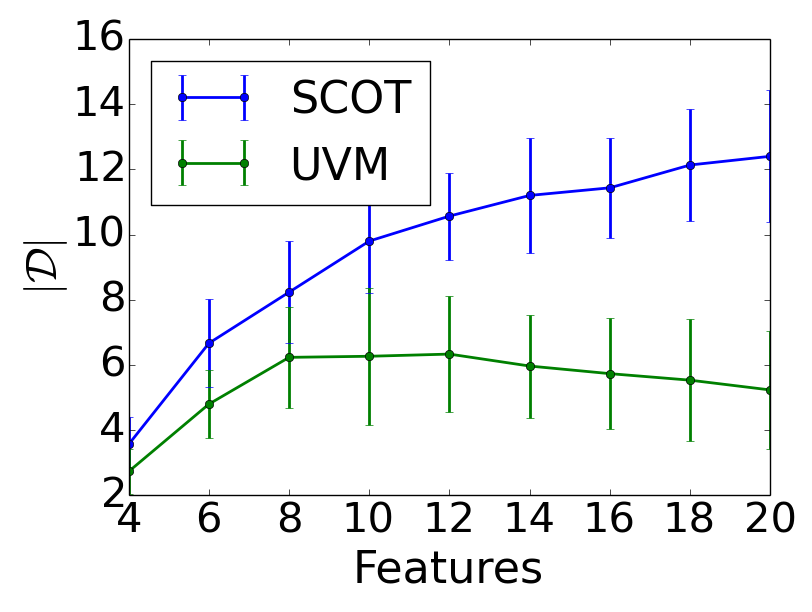}
\label{subfig:sensitivityFeatures}
}
\subfigure[]{
\includegraphics[scale=0.3]{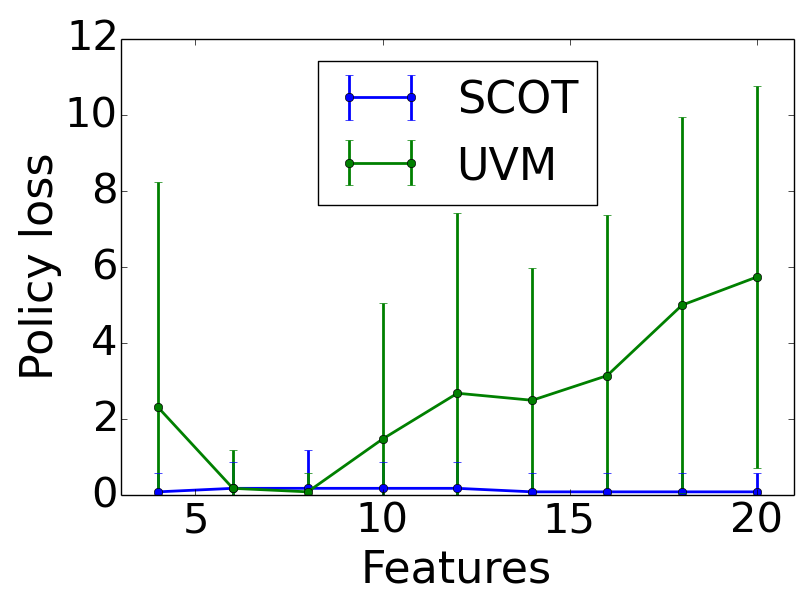}
\label{subfig:sensitivityFeatures}
}
\caption{SCOT algorithm is robust to increasing numbers of features, whereas UVM is not robust since it relies on good volume estimates. Results averaged over 30 runs on a 6x6 grid world. UVM uses 1,000,000 samples for volume estimation.}
\label{fig:uvm_vs_scot}
\end{figure}

\section{Active learning experiment parameters}
We generated 10x10 random grid worlds where each state was assigned a random one-hot 10-dimensional feature vector we set the discount factor to $\gamma = 0.95$. We ran BIRL with a uniform prior to obtain the MAP reward function given demonstrations for each active IRL algorithm. Each active query resulted in an optimal trajectory of length 20 demonstrated from the optimal policy. We ran the MCMC chain for 10,000 steps using $\alpha = 100$ and step size of $0.005$.

\section{BIO-IRL algorithm specifics}
We calculate angular similarity as follows: We take all the normal vectors from BEC($\mathcal{D}|\pi^*$) and do a greedy matching to BEC($\pi^*$). Once we match the first vector in BEC($\mathcal{D}|\pi^*$) with the closest vector in BEC($R$), we remove the best match from BEC($\pi^*$) and continue with the next vector in BEC($\mathcal{D}|\pi^*$).  When there are no remaining half-spaces in BEC($\mathcal{D}|\pi^*$) The algorithm returns the cumulative sum of half-space similarities divided by the number of half-spaces in BEC($\pi^*$).

Because our normal vectors can have positive and negative elements, we define the similarity between two vectors $\mathbf{x}$ and $\mathbf{y}$ as 
\begin{equation}
similarity(\mathbf{x}, \mathbf{y}) = 1 - \cos^{-1}(\mathbf{x} \cdot \mathbf{y}) / \pi.
\end{equation} 

\subsection{Markov chain BIO-IRL experiment}
We used $\alpha = 100$ as the softmax temperature parameter for BIRL and BIO-IRL.

\subsection{Ball sorting BIO-IRL experiment}
We discretized the table top into a 6x6 grid of positions, all of which are potential starting states. The four discrete actions move the ball along the table top in the four cardinal directions. The demonstrations consisted of optimal trajectories found using Value Iteration of length 10. BIRL and BIO-IRL both used the following parameters: $\alpha=100$, MCMC chain length=1000, MCMC step size = 0.05. BIO-IRL used $\lambda = 100$.

\end{document}